\def\year{2022}\relax
\documentclass[letterpaper]{article} 
\usepackage{aaai22}  
\usepackage{times}  
\usepackage{helvet}  
\usepackage{courier}  
\usepackage[hyphens]{url}  
\usepackage{graphicx} 
\usepackage{natbib}  
\usepackage{caption} 
\DeclareCaptionStyle{ruled}{labelfont=normalfont,labelsep=colon,strut=off} 
\usepackage{algorithm}
\usepackage{algorithmic}

\usepackage{newfloat}
\usepackage{listings}
\floatstyle{ruled}
\newfloat{listing}{tb}{lst}{}
\floatname{listing}{Listing}
\title{A Unifying Theory of Thompson Sampling\\ for Continuous Risk-Averse Bandits}
\author{
    Joel Q. L. Chang\textsuperscript{\rm 1},
    Vincent Y. F. Tan\textsuperscript{\rm 1, 2}
}
\affiliations{
    \textsuperscript{\rm 1}Department of Mathematics, National University of Singapore\\
        \textsuperscript{\rm 2}Department of Electrical and Computer Engineering, National University of Singapore\\
    joel.chang@u.nus.edu, vtan@nus.edu.sg
}

\usepackage{bibentry}

\usepackage{tikz}
\usepackage{pgfplots}

\definecolor{blue1}{RGB}{0, 0, 200}
\definecolor{blue2}{RGB}{0, 0, 220}
\definecolor{blue3}{RGB}{0, 0, 130}
\definecolor{red2}{RGB}{180,0,0}
\definecolor{purp1}{RGB}{200,50,0}
\definecolor{green1}{RGB}{0,130,0}
\definecolor{lightgray}{RGB}{190,190,190}

\usepackage{amssymb, amsmath, amsthm}

\usepackage{pifont}
\newcommand{\xmark}{\text{\ding{55}}}

\newtheorem{theorem}{Theorem}
\newtheorem{corollary}{Corollary}
\newtheorem{lemma}{Lemma}

\newtheorem{proposition}{Proposition}
\newtheorem{definition}{Definition}

\newtheorem{remark}{Remark}
\newtheorem{example}{Example}

\newcommand{\paren}[1]{\left({#1}\right)}
\newcommand{\parenb}[1]{\left[{#1}\right]}
\newcommand{\parenl}[1]{\left|{#1}\right|}

\newcommand{\RR}{\mathbb{R}}
\newcommand{\NN}{\mathbb{N}}

\newcommand{\eeps}{\varepsilon}
\newcommand{\To}{\Rightarrow}

\newcommand{\norm}[1]{\parenl{\myvec{{#1}}}}
\newcommand{\del}{\partial}

\newcommand{\inv}[1]{#1^{-1}}

\renewcommand{\iff}{\Leftrightarrow}

\newcommand{\disp}{}

\newcommand{\rint}[4]{\disp \int_{#1}^{#2} {#3}\ \mathrm{d}{#4}}
\newcommand{\Fol}{\sim}

\newcommand{\seq}[1]{({#1})}

\renewcommand{\norm}[1]{\left\|{#1}\right\|}

\newcommand{\PP}{\mathbb{P}}
\newcommand{\EE}{\mathbb{E}}
\newcommand{\II}{\mathbb{I}}

\newcommand{\argmax}{\mathrm{arg\ max}}
\newcommand{\ccal}[1]{\mathcal{#1}}

\newcommand{\Dir}{\mathrm{Dir}}

\newcommand{\KL}{\mathrm{KL}}

\newcommand{\vocab}[1]{{\textit{#1}}}

\usepackage{hyperref}
\hypersetup{ %
    pdftitle={},
    pdfauthor={},
    pdfkeywords={},
    pdfborder=0 0 0,
    pdfpagemode=UseNone,
    colorlinks=true,
    linkcolor=blue1,
    citecolor=blue1,
    filecolor=blue1,
    urlcolor=black,
    pdfview=FitH
    }

\renewcommand{\leq}{\leqslant}
\renewcommand{\geq}{\geqslant}

\begin{document}

\maketitle
\begin{abstract}
\label{abstract}
This paper unifies the design and the analysis of risk-averse Thompson sampling algorithms for the multi-armed bandit problem for a class of risk functionals $\rho$ that are \vocab{continuous and dominant}.
We prove generalised concentration bounds for these continuous and dominant risk functionals and show that a wide class of popular risk functionals belong to this class.
Using our newly developed analytical toolkits, we analyse the algorithm $\rho$-MTS (for multinomial distributions) and prove that they admit asymptotically optimal regret bounds of risk-averse algorithms under the CVaR, proportional hazard, and other ubiquitous risk measures.
More generally, we prove the asymptotic optimality of $\rho$-MTS for Bernoulli distributions for a class of risk measures known as empirical distribution performance measures (EDPMs); this includes the well-known mean-variance.
Numerical simulations show that the regret bounds incurred by our algorithms are reasonably tight vis-à-vis algorithm-independent lower bounds.
\end{abstract}

\section{Introduction}
\label{submission}
Consider a $K$-armed multi-armed bandit (MAB) with unknown distributions $\nu = {(\nu_k)}_{k \in [K]}$ called \vocab{arms} and a time horizon $n$. At each time step $t \in [n]$, a learner chooses an arm $A_t \in [K]$ and obtains a random reward $X_{A_t}$ from the corresponding distribution $\nu_{A_t}$. In the vanilla MAB setting, the learner aims to maximise their expected total reward after $n$ selections, requiring a strategic balance of exploration and exploitation of the arms.
Much work has been developed in this field for L/UCB-based algorithms, and in recent developments, more Thompson sampling-based algorithms have been designed and proven to attain the theoretical asymptotic lower bounds that outperform their L/UCB-based counterparts.
However, many real-world settings include the presence of risk, which precludes the adoption of the mean-maximisation objective. Risk-averse bandits address this issue for bandit models by replacing the expected value by some measure of risk.

Recent work has incorporated risk into the analysis, with different works working with different risk measures that satisfy various properties.
In the existing literature, the more popular risk measures being considered are mean-variance \citep{sani2013riskaversion, zhu2020thompson} and conditional value-at-risk (CVaR) \citep{tamkin2020dist, kyn2020cvar, baudry2020thompson, chang2021riskconstrained}. 
In particular, CVaR is an instance of a general class of risk functionals, called \vocab{coherent risk functionals} \citep{Artzner}.
\citet{huang2021offpolicy} observed that for nonnegative rewards, coherent risk functionals are subsumed in broader class of functionals called \vocab{distorted risk fuctionals}.
Common distorted risk functionals, such as the expected value and CVaR, satisfy theoretically convenient continuity properties.

However, not much work has been done to \emph{unify} these various risk-averse algorithms to elucidate the common machinery that underlie them.
In this paper, we provide one way to unify these risk-averse Thompson sampling algorithms through \emph{continuous and dominant risk functionals}, which we denote by $\rho$.
We design two Thompson sampling-based algorithms---$\rho$-MTS and $\rho$-NPTS---to solve the modified MABs, provably achieving asymptotic optimality for $\rho$-MTS under a variety of risk functionals and empirically doing so for $\rho$-NPTS.
Therefore, we unify much of the progress made in analysing Thompson sampling-based solutions to risk-averse MABs.

\subsection{Related Work}
\label{related}
\citet{thompson1933likelihood} proposed the first Bayesian algorithm for MABs known as Thompson sampling.
\citet{LAI19854} proved a lower bound on the regret for any instance-dependent bandit algorithm for the vanilla MAB.
\citet{kaufmann2012thompson, pmlr-v23-agrawal12} analysed the Thompson sampling algorithm to solve the $K$-armed MAB for Bernoulli and Gaussian reward distributions respectively, and proved the asymptotic optimality in the Bernoulli setting relative to the lower bound given by \citet{LAI19854}.
\citet{4724951} proposed the Bayesian learning automaton that is self-correcting and converges to only pulling the optimal arm with probability $1$.
\citet{pmlr-v117-riou20a} designed and proved the asymptotic optimality of Thompson sampling on bandits which firstly follow multinomial distributions, followed by general bandits that are bounded in $[0,1]$ by discretising $[0,1]$ and using suitable approximations on each sub-interval.

Many variants of the MAB which factor risk have been considered. One popular risk measure is mean-variance. \citet{sani2013riskaversion} proposed the first U/LCB-based algorithm called MV-UCB to solve the mean-variance MAB problem.
\citet{valiki15} tightened the regret analysis of MV-UCB, establishing the order optimality of MV-UCB.
\citet{zhu2020thompson} designed and analysed the first risk-averse mean-variance bandits based on Thompson sampling which follow Gaussian distributions, providing novel tail upper bounds and a unifying framework to consider Thompson samples with various means and variances.
\citet{du2021continuous_neurips} further generalised this problem, considering continuous mean-covariance linear bandits, which specialises into the stochastic mean-variance MAB in the $1$-dimensional setting.

Another popular risk measure is \vocab{Conditional Value-at-Risk} (abbreviated as CVaR).
\citet{Galichet13} designed the L/UCB-based Multi-Armed Risk-Aware Bandit (\textsc{MaRaB}) algorithm to solve the CVaR MAB problem.
\citet{chang2021riskconstrained} and \citet{baudry2020thompson} contemporaneously designed and analysed Thompson sampling algorithms for the risk measure CVaR. The former proved near-asymptotically optimal regret bounds for Gaussian bandits, and the latter proved asymptotically optimal regret bounds for rewards in $[0,1]$ by judiciously analysing the compact spaces induced by CVaR and designing and proving new concentration bounds.

Other generalised frameworks of risk functionals have also been studied.
\citet{wang_1996} studied distorted risk functionals that generalise the expectation and CVaR, characterising the risk functionals by their distortion functions that are non-decreasing on $[0,1]$.
\citet{pmlr-v75-cassel18a} analysed 
empirical distribution performance measures (EDPMs), which are by definition continuous on the (Banach) space of bounded random variables under the uniform norm. In Table 1 therein, these EDPMs provide the interface for many instances of other popular risk functionals, such as second moment, entropic risk, and Sharpe ratio.
\citet{NEURIPS2020_9f60ab2b} studied risk-sensitive learning schemes by rejuvenating the notion of optimized certainty equivalents (OCE), which subsumes common risk functionals like expectation, entropic risk, mean-variance, and CVaR. 
\citet{huang2021offpolicy} defined  Lipschitz risk functionals which subsume many of these common risk measures under suitable smoothness assumptions; these include variance, mean-variance, distorted risk functionals, and Cumulative Prospect Theory-inspired (CPT) risk functionals.
\subsection{Contributions}
\label{contributions}
\begin{itemize}
\item We   present the key properties that any continuous and dominant risk functional (Definition~\ref{def: continuous_risk_functional}) $\rho$ possesses that are then exploited in the regret analysis of the Thompson sampling algorithms. This provides the theoretical underpinnings for our proposed Thompson sampling-based algorithms to solve any $\rho$-MAB problem.
\item We state and prove new upper and lower tail bounds for $\rho$ on multinomial distributions, generalising and unifying the underlying theory for the upper and lower bounds obtained in \citet{pmlr-v117-riou20a} and \citet{baudry2020thompson}. These new tail bounds generalise the risk functional beyond expected value \citep{pmlr-v117-riou20a} and CVaR \citep{baudry2020thompson}, to apply to continuous and dominant risk functionals.
\item We also design two Thompson sampling-based algorithms: $\rho$-MTS for bandits on multinomial distributions and $\rho$-NPTS for bandits on distributions whose rewards are bounded in any compact subset $C \subset \RR$. 
We show that for many continuous and dominant risk functionals $\rho$, $\rho$-MTS is asymptotically optimal.
Setting $\rho$ to common risk measures, we recover asymptotically optimal algorithms for the respective $\rho$-MAB problems \citep{pmlr-v117-riou20a, zhu2020thompson, baudry2020thompson}, and significantly improve on the regret bounds for Bernoulli-MVTS in \citet{zhu2020thompson}; see Remark~\ref{rmk:compare}.
\end{itemize}

\subsection{Preliminaries}
Let $\NN$ be the set of positive integers.
For any $M \in \NN$, define $[M] = \{1,\dots,M\}$ and ${[M]}_0 = [M] \cup \{0\}$.
For any $M \in \NN$, denote the $M$-probability simplex as $\Delta^M := \{p \in {[0,1]}^{M+1} : \sum_{i \in {[M]}_0} p_i=1\}$. 
For any $p,q \in \Delta^M$, we denote the $\ell_\infty$ distance between them as $$d_\infty(p,q) := \max_{i \in {[M]}_0} |p_i-q_i|.$$

Before formally stating the problem, we need to introduce some measure-theoretic and topological notions which will be essential in the analysis.

Fix a compact subset $C \subseteq \RR$.
Then $(C,|\cdot|)$ is a separable metric space with Borel $\sigma$-algebra denoted by ${\frak B}(C)$, constituting the measurable space $(C, {\frak B}(C))$.
For each $c \in C$, let $\delta_{c} := \II\{c \in \cdot\}$ denote the Dirac measure at $c$.

Let ${\ccal P}$ denote the collection of probability measures on $(C,{\frak B}(C))$.
Each $\mu \in {\ccal P}$ admits a cumulative distribution function (CDF) $F_\mu = \mu((-\infty,\cdot]) : C \to [0,1]$. Hence,  on~${\ccal P}$, we can define the \vocab{Kolmogorov-Smirnov metric} $$D_{\infty} : (\mu,\eta) \mapsto \sup_{t \in C} |F_\mu(t) - F_{\eta}(t)|.$$ We can also define the \vocab{L\'evy-Prokhorov metric} 
\begin{align*}
D_{\mathrm{L}} : (\mu,\eta)& \mapsto \inf \{\eeps > 0: \\
&\hspace{-.4in} F_\mu(x-\eeps)-\eeps \leq F_{\eta}(x) \leq F_\mu(x+\eeps) + \eeps,\forall\,  x \in \RR\}
\end{align*}on ${\ccal P}$.
Thus, $({\ccal P},d)$ is a metric space in either metric $d \in \{D_\infty,D_{\mathrm{L}}\}$.
For any $\mu,\eta \in {\ccal P}$, let $\KL(\mu,\eta) := \rint{C}{}{\log(\mathrm{d}\mu/\mathrm{d}\eta)}{\mu}$ denote the relative entropy or  Kullback-Leibler (KL) divergence between $\mu$ and $\eta$.

We will now provide three examples of compact metric subspaces $(\ccal C,d)$ of $(\ccal P,d)$ which we will utilise in our algorithms and lemmas therein.
\begin{example}[$({\ccal P}_{\ccal S}, D_\infty)$]
	{\em We first consider $({\ccal P}_{\ccal S}, D_\infty)$---the set of probability mass functions on a finite alphabet $\ccal S = \{s_0,\dots,s_M\} \subset C$ under the $D_\infty$ metric. For each $p \in \Delta^M$, define $\mu_p = \sum_{i=0}^M p_i \delta_{s_i}$, and ${\frak D}_{\ccal S} : \Delta^M \to {\ccal P}$ by $p \mapsto \mu_p$. Then ${\frak D}_{\ccal S}$ is an imbedding into ${\ccal P}$ due to  the inequality $d_\infty(p,q) \leq 2 D_{\infty}({\frak D}_{\ccal S}(p),{\frak D}_{\ccal S}(q)) \leq 2M d_\infty(p,q).$ This implies that $({\ccal C},d) := ({\frak D}_{\ccal S}(\Delta^M), D_\infty)$ is a compact metric space. For brevity, we denote ${\ccal P}_{\ccal S} :={\frak D}_{\ccal S}(\Delta^M)$.}
\end{example}
\begin{example}[$({\ccal P}, D_{\mathrm{L}})$]
	{\em By \citet[Theorem 1.12]{bams/1183523140}, ${\ccal P}$ is a compact set in the topology of weak convergence, which is metrized by the L\'evy-Prokhorov metric $D_{\mathrm{L}}$ on ${\ccal P}$. This implies that $({\ccal P}, D_{\mathrm{L}})$ is a compact metric space. Furthermore, by \citet{1055416}, $\KL(\cdot,\cdot)$ is jointly lower-semicontinuous in both arguments.
	}	
\end{example}
\begin{example}[$({\ccal P}_{\mathrm{c}}^{(B)}, D_{\mathrm{L}})$]
{\em This is the set of probability measures whose CDFs have continuous derivatives that are uniformly bounded by $B$, i.e., ${\ccal P}_{\mathrm{c}}^{(B)}:= \{\mu \in {\ccal P}: \text{$F_\mu'$ is cts on $C$ and $\sup_{c\in C}|F_\mu'(c)| \leq B$}\}$. By the Arzel\`a-Ascoli Theorem, $({\ccal P}_{\mathrm{c}}^{(B)},D_\infty) \subseteq ({\ccal P},D_\infty)$ is compact and thus as topological spaces $({\ccal P}_{\mathrm{c}}^{(B)}, D_{\mathrm{L}})=({\ccal P}_{\mathrm{c}}^{(B)},D_\infty)$ is a compact metric space.
}
\end{example}
Thus, we let $({\ccal C},d)$ denote any compact metric subspace of $(\ccal P,d)$, of which includes $({\ccal P}_{\ccal S}, D_\infty)$, $({\ccal P}, D_{\mathrm{L}})$, and $({\ccal P}_{\mathrm{c}}^{(B)}, D_{\mathrm{L}})$. Since $C$ is closed and bounded, we can assume without loss of generality that $C \subseteq [0,1]$ by rescaling.

Let ${\ccal L}_\infty$ denote the space of $C$-valued bounded random variables.
In particular, we do not place restrictions on the probability space that each $X \in {\ccal L}_\infty$ is defined on.
\begin{definition}\label{def: risk functional}
	{\em A \vocab{risk functional} is an $\RR$-valued map $\rho : {\ccal P} \to \RR$ on ${\ccal P}$. A \vocab{conventional risk functional} $\varrho : {\ccal L}_\infty \to \RR$ is an $\RR$-valued map on ${\ccal L}_\infty$.}
\end{definition}
A conventional risk functional $\varrho : {\ccal L}_\infty \to \RR$ is said to be \vocab{law-invariant} \citep{huang2021offpolicy} if for any pair of $C$-valued random variables $X_i : (\Omega_i,{\ccal F}_i, \PP_i) \to (C,{\frak B}(C))$ with probability measures $\mu_{i} := \PP_i \circ \inv X_i \in {\ccal P}$, $i=1,2$, $$\mu_{1} = \mu_{2} \To \varrho(X_1) = \varrho(X_2).$$
\begin{remark}\label{rmk: well_defined_rho}
{\em We demonstrate in the first section of the supplementary material that $\rho$ is indeed well-defined. That is, for any random variable $X$ sampled from a probability measure $\mu$ and law-invariant conventional risk functional $\varrho$, we can write $\rho(\mu) = \varrho(X)$ without ambiguity. However, we consider it more useful to assume $\rho$ whose domain is a metric space $(\ccal P,d)$, since we can apply the topological results of $(\ccal P,d)$ in the formulation of our concentration bounds.}
\end{remark}

\subsection{Paper Outline}
In the following, we first define continuous and dominant risk functionals, and state some essential properties and crucial concentration bounds that guarantee the asymptotic optimality guarantee for $\rho$-MTS.
We also provide examples of many popular risk functionals that satisfy the proposed notion of a continuous and dominant risk functional.
We then formally define the risk-averse $\rho$-MAB problem, and design two Thompson sampling-based algorithms $\rho$-MTS and $\rho$-NPTS to solve this problem.
Finally, we state our derived regret bound for $\rho$-MTS and provide a proof outline of the key ideas involved therein, thus demonstrating the asymptotic optimality of $\rho$-MTS.
This significantly generalises existing work on Thompson sampling for MABs with finite alphabets to many popular risk functionals used in practice.

%

\section{Continuous Risk Functionals}
\label{continuous_risk_functionals}
In this section, we define continuous risk functionals, which are the risk measures of interest in our Thompson sampling algorithms. We demonstrate that when $\rho$ is continuous and dominant~(see Definition~\ref{def: dominant_risk_functional}), its corresponding $\rho$-MTS and $\rho$-NPTS algorithms achieve the asymptotically optimal regret bound, the former provably and the latter empirically.
\begin{table*}[t]
	\begin{center}
	\begin{tabular}{|c||c|c|c|}
		\hline
		&&&\\[-.6em]
		Distorted risk functional & Definition of $\rho_g(\mu) = \varrho_g(X)$ & $g(x)$ & Continuity of $\rho_g$\\[.5em]\hline
		&&&\\[-.6em]
		Expectation ($\EE$) & $\EE[X]$ & $x$ & $\checkmark$\\[.2em]
		CVaR ($\mathrm{CVaR}_\alpha$) & $-\frac{1}{\alpha}\rint{0}{\alpha}{\mathrm{VaR}_\gamma(X)}{\gamma}$ & $\min \{x/(1-\alpha),1\}$ & $\checkmark$\\[.2em]
		Proportional hazard ($\mathrm{Prop}_p$) & $\rint{0}{\infty}{{(S_X(t))}^p}{t}$ & $x^p$ & $\checkmark$\\[.2em]
		Lookback ($\mathrm{LB}_q$)  & $\rint{0}{\infty}{({S_X(t))}^q(1-q\log {S_X(t)})}{t}$ & $x^q(1-q\log x)$ & $\checkmark$\\[.2em]
		VaR ($\mathrm{VaR}_\alpha$) & $-\inf \{x \in \RR : F_X(x) > \alpha\}$ & $\II\{x \geq 1-\alpha\}$ & $\xmark$\\[.2em]\hline
	\end{tabular}
	\caption{A table of common distorted risk functionals, where $S_X(t) := 1 - F_X(t)$ denotes the \vocab{decumulative} distribution function \citep{wang_1996}. The risk functionals indicated by $\checkmark$ admit the asymptotically optimal result in Theorem~\ref{thm: rho_mts_upper_bound}.
	}
	\label{table: common_drf}
	\end{center}
\end{table*}
\begin{table*}[t]
	\begin{center}
	\begin{tabular}{|c||c|c|}
		\hline
		&&\\[-.6em]
		EDPM & $U : \ccal P \to \RR$ & Convexity\\[.5em]\hline
		&&\\[-.6em]
		Expectation & $U^{\mathrm{ave}}(\nu) := \EE[X]$ & $\checkmark$\\[.2em]
		Second moment & $U^{\EE^2}(\nu) := \EE[X^2]$ & $\checkmark$\\[.2em]
		Below target semi-variance & $U^{-\mathrm{TSV}_r}(\nu) := -\EE[{(X-r)}^2\II\{X\leq r\}]$ & $\checkmark$\\[.2em]
		Entropic risk & $U^{\mathrm{ent}_\theta}(\nu) := -\frac{1}{\theta} \log(\EE[-\theta X])$ & $\checkmark$\\[.2em]
		Negative variance & $U^{-\sigma^2}(\nu) := -(U^{\EE^2}(\nu) - {(U^{\mathrm{ave}}(\nu))}^2)$ & $\checkmark$\\[.2em]
		Mean-variance & $U^{\mathrm{MV}_\rho}(\nu) := \gamma\, U^{\mathrm{ave}}(\nu) + U^{-\sigma^2}(\nu)$ & $\checkmark$\\[.2em]
		Sharpe ratio & $U^{\mathrm{Sh}_r}(\nu) := \frac{U^{\mathrm{ave}}(\nu)-r}{\sqrt{\eeps_\sigma - U^{-\sigma^2}(\nu)}}$ & $\xmark$\\[.2em]
		Sortino ratio & $U^{\mathrm{So}_r}(\nu) := \frac{U^{\mathrm{ave}}(\nu)-r}{\sqrt{\eeps_\sigma - U^{-\mathrm{TSV}_r}(\nu)}}$ & $\xmark$\\[.2em]\hline
	\end{tabular}
	\caption{A table of EDPMs, where $U : \ccal P \to \RR$ characterises each risk functional \citep{pmlr-v75-cassel18a}.
	When $M=1$ (i.e. in the case of Bernoulli bandits), the non-asymptotic lower bound in Lemma~\ref{lem: lower_bound} and the asymptotically optimal result in Theorem~\ref{thm: rho_mts_upper_bound} hold for risk functionals indicated by $\checkmark$; see Remark~\ref{rmk:compare}.}
	\label{table: common_EDPM}
	\end{center}
\end{table*}
\begin{definition}[Continuous Risk Functional]\label{def: continuous_risk_functional}
{\em Let ${\ccal P}$ be equipped with the metric $d$. A risk functional $\rho$ is said to be
\vocab{continuous at $\mu \in {\ccal P}$} if for any $\eeps > 0$, there exists $\delta > 0$, which may depend on $\mu \in {\ccal P}$, such that
\begin{equation}\label{eqn: continuity_of_rho}
d(\mu,\eta) < \delta \To \parenl{\rho(\mu)-\rho(\eta)} < \eeps.	
\end{equation}
We say that $\rho$ is \vocab{continuous} on a subset $\ccal Q \subseteq \ccal P$ if it is continuous at every $\mu \in \ccal Q$. We say that $\rho$ is \vocab{uniformly continuous} on ${\ccal Q}$ if for any $\eeps > 0$, there exists $\delta > 0$ that does not depend on $\mu \in {\ccal Q}$, such that (\ref{eqn: continuity_of_rho}) holds.}
\end{definition}
We also remark that for any compact metric subspace $({\ccal C},d) \subseteq ({\ccal P},d)$ and continuous risk functional $\rho$, $\rho|_{{\ccal C}}$ is uniformly continuous on $({\ccal C},d)$.

Let $({{\ccal C}},d)$ be any of the three compact metric spaces $({\ccal P}_{\ccal S},D_\infty)$, $ (\ccal P,D_{\mathrm{L}})$, $({\ccal P}_{\mathrm{c}}^{(B)},D_{\mathrm{L}})$.
For any risk functional $\rho: \ccal{P} \to \RR$, define 
\begin{align*}
{\ccal K}_{\inf}^\rho(\mu,r) &:= \inf_{\eta \in {{\ccal C}}}\{\KL(\mu,\eta) : \rho(\eta) \geq r\}.
\end{align*}
In the case $(\ccal C, d) = ({\ccal P}_{\ccal S}, D_\infty)$ for some fixed alphabet $\ccal S$, define $\sigma_{\rho,\ccal S} := \rho \circ {\frak D}_{\ccal S} : \Delta^M \to \RR$. We note that $\sigma_{\rho,\ccal S}$ is continuous on $\Delta^M$ if $\rho$ is continuous on $({\ccal P}_{\ccal S},D_\infty)$.

By Lemma~18 in \citet{pmlr-v117-riou20a} $\rho$ being continuous on $({\ccal P},D_{\mathrm{L}})$ implies its continuity on $({\ccal P},D_\infty)$, and $\rho$ being continuous on $({\ccal P}_{\mathrm{c}}^{(B)}, D_\infty)$ implies its continuity on $({\ccal P}_{\mathrm{c}}^{(B)}, D_{\mathrm{L}})$. This conclusion is consistent with that in \citet{baudry2020thompson} whose B-CVTS algorithm assumes that the reward distributions are continuous.
\subsubsection{Tail Upper Bound}
Risk functionals $\rho$ that are continuous on ${\ccal P}_{\ccal S}$ satisfy a generalization of the tail upper bound developed by \citet{pmlr-v117-riou20a}.
\begin{lemma}\label{lem: upper_bound}
	Let $\rho : \ccal P \to \RR$ be a risk functional continuous on $({\ccal P}_{\ccal S},D_\infty)$ for some finite alphabet $\ccal S$ of size $M+1$, and $r \in \RR$. Fix $\alpha \in \NN^{M+1}$, $n= \sum_{i=0}^M \alpha_i$, and $p = \alpha/n$.
	Then for any random variable $L \sim \Dir(\alpha)$,
	\begin{align*}
	\PP(\sigma_{\rho,\ccal S}(L) \geq r) &\leq C_1n^{M/2} \exp(-n {\ccal K}_{\inf}^{{\rho}_{\ccal S}}({\frak D}_{\ccal S}(p),r));\\
	\PP(\sigma_{\rho,\ccal S}(L) \leq r) &\leq C_1n^{M/2} \exp(-n {\ccal K}_{\inf}^{{\rho}_{\ccal S}}({\frak D}_{\ccal S}(p),r)),
	\end{align*}
	where $C_1 := {\Gamma(M+1)}^{-1}{(2\pi)}^{-M/2}e^{1/12}$.
\end{lemma}
We remark that Lemma~\ref{lem: upper_bound} generalises the upper bound in \citet[Lemma~13]{pmlr-v117-riou20a} to risk functionals that are ``sufficiently continuous''.
\begin{proposition}\label{cor: continuity_of_KL}
Let $\rho : {\ccal P} \to \RR$ be a continuous risk functional. Then the mapping $\ccal K_{\inf}^{\rho} : {\ccal P} \times \rho({\ccal C}) \to \RR$ is lower-semicontinuous in both of its arguments.
\end{proposition}
\subsubsection{Examples of Continuous Risk Functionals} We provide numerous examples of  risk functionals that satisfy the proposed notion of continuity.
\begin{definition}[Distorted Risk Functional]\label{def: distorted_risk_functional}
{\em Let $C=[0,D]$ and $X$ be a $C$-valued random variable sampled from a probability measure $\mu \in {\ccal P}$ and CDF $F_\mu$ its corresponding CDF. A conventional risk functional is said to be a \vocab{distorted risk functional} \citep{wang_1996, huang2021offpolicy} if there exists a non-decreasing function $g : [0,1] \to [0,1]$, called a \vocab{distortion function}, satisfying $g(0) = 0$ and $g(1) = 1$ such that
\begin{equation}\label{eqn: distorted_risk_functional_bounded}
	\varrho_g(X) = \rint{0}{D}{g(1-F_\mu(t))}{t}.
\end{equation}
We append the subscript $g$ to $\varrho$ and write $\varrho_g$ to emphasise the distorted function $g$ associated with $\rho$.
By definition, distorted risk functionals are law-invariant.
By Remark~\ref{rmk: well_defined_rho}, we can write $\rho_g(\mu) \equiv \varrho_g(X)$ thereafter and consider distorted risk functionals $\rho_g$ whose domain is ${\ccal P}$.}
\end{definition}
\begin{proposition}\label{prop: distorted_risk_cty}
	Suppose $g$ is continuous on $[0,1]$.
	Then the distorted risk functional $\rho_g : {\ccal P} \to \RR$ is continuous on $({\ccal P},D_\infty)$. Consequently, $\rho_g$ is continuous on $({\ccal P}_{\mathrm{c}}^{(B)},D_{\mathrm{L}})$.
\end{proposition}
\begin{example}
	{\em Table~\ref{table: common_drf} lists some commonly used distorted risk functionals and the properties that they satisfy.}
\end{example}
\begin{corollary}\label{cor: cts_drfs}
	On the space of rewards in $C$, the risk functionals expected value, $\mathrm{CVaR}_\alpha$, proportional hazard, and Lookback as defined in Table~\ref{table: common_drf} are continuous on $({\ccal P},D_\infty)$. 
\end{corollary}
Similar arguments can be used to show that the Cumulative Prospect Theory-Inspired (CPT) functionals \citep{huang2021offpolicy}, are also continuous on $({\ccal P},D_\infty)$. Nevertheless, we remark that $\mathrm{VaR}_\alpha$ (last row of Table~\ref{table: common_drf}) is not  continuous on $({\ccal P},D_\infty)$, and thus, does not necessarily enjoy the regret bounds from $\rho$-MTS.
\begin{example}\label{eg: EDPM_cty}
	{\em Table~\ref{table: common_EDPM} lists some commonly used  EDPMs, their distortion functions, and the (convexity) properties that they satisfy. By their formulation in \citet{pmlr-v75-cassel18a}, they are all continuous on $(\ccal P,D_\infty)$.}
\end{example}
\begin{remark}
\label{rmk: cts_combinations}	
{\em We observe that for scalars $\lambda_1,\dots,\lambda_n \in \RR$ and continuous risk functionals $\rho_1,\dots,\rho_n$ on $(\ccal P,d)$, the linear combination $\sum_{i=1}^n \lambda_i \rho_i$ is a continuous risk functional on $(\ccal P,d)$.
Furthermore, for any continuous function $\phi : \RR \to \RR$ and continuous risk functional $\rho$, the composition $\phi \circ \rho$ is also a continuous risk functional. This allows us to consider many \emph{combinations} of risk functionals.}
\end{remark}
\begin{example}[Continuity of Linear Combinations]
\label{eg: cty_of_linear_combi}
	{\em For instance, consider the risk functionals ${\mathrm{MV}}_\gamma$, $ {\mathrm{CVaR}_\alpha}$, ${\mathrm{Prop}}_{p}$, ${\mathrm{LB}}_{q}$ for fixed parameters $\gamma > 0$, $\alpha \in [0,1)$, $p \in (0,1)$, $q\in(0,1)$. By Example~\ref{eg: EDPM_cty} and Corollary~\ref{cor: cts_drfs}, these risk functionals are continuous on $({\ccal P},D_{\infty})$, and the risk functionals $\rho_1 := {\mathrm{MV}}_\gamma + {\mathrm{CVaR}_\alpha}$ and $\rho_2 := {\mathrm{Prop}}_{p} + {\mathrm{LB}}_{q}$ are continuous on $({\ccal P},D_{\infty})$, and consequently, are continuous on $({\ccal P}_{\mathrm{c}}^{(B)},D_{\mathrm{L}})$. Thus, innumerable risk functionals can be synthesised (as will be done in the section on numerical experiments) and our Thompson sampling-based algorithms are not only applicable, but also empirically competitive with the theoretical lower bound.}
\end{example}
\section{Dominant Risk Functionals}
We now introduce the notion of \vocab{dominant risk functionals}. These risk functions  admit  a crucial tail lower bound (stated in Lemma~\ref{lem: lower_bound} to follow).
Let $\rho : \ccal P \to \RR$ be a risk functional, $\ccal S$ a finite set with size $M+1$, and denote $\rho_{\ccal S} := \rho|_{{\ccal P}_{\ccal S}}$ and
${\ccal T}_{\rho,\ccal S}(r) := \inv{\rho_{\ccal S}}([r,\infty))$.
For any $p \in \Delta^M$, define ${\ccal T}_{\rho,\ccal S, p}^{(1)} := {\ccal T}_{\rho,\ccal S}({\sigma}_{\rho,\ccal S}(p))$.
	 For any $\ccal I \subseteq {[M]}_0$, define
	 $${\ccal T}_{\ccal S, p}^{(2)}(\ccal I) := \left\{{\frak D}_{\ccal S}(q) \in {\ccal P}_{\ccal S}: q_i \in \begin{cases} [0,p_i] & \mbox{for $i \in \ccal I$} \\ [p_i,1] & \mbox{for $i \notin \ccal I$} \end{cases}\right\}.$$
	 We have ${\ccal T}_{\ccal S,p}^{(2)}(\emptyset) = {\ccal T}_{\ccal S,p}^{(2)}({[M]}_0)=\{{\frak D}_{\ccal S}(p)\}$.
	\begin{definition}\label{def: dominant_risk_functional}
	{\em We say that a risk functional $\rho : \ccal P \to \RR$ is \vocab{dominant} if for any finite alphabet $\ccal S$ and $r \in \RR$, there exists $\mu_* = {\frak D}_{\ccal S}(p_*)$ with
	\begin{equation}\label{eqn: infimum_condition}
		{\ccal K}_{\inf}^{\rho}(\mu,r) = \KL(\mu,\mu_*)
	\end{equation}
	 and $\ccal I \subseteq {[M]}_0, \ccal I \neq \emptyset,{[M]}_0$, such that
	 \begin{equation} 
		{\ccal T}_{\rho,\ccal S, p_*}^{(1)} \supseteq {\ccal T}_{\ccal S, p_*}^{(2)}(\ccal I). \nonumber
	\end{equation}
	We remark that if $\rho$ is continuous on ${\ccal P}_{\ccal S}$, then it satisfies (\ref{eqn: infimum_condition}) by the Extreme Value Theorem.}
\end{definition}
 We illustrate the property of $\rho$ being dominant in Figure~\ref{fig: dom_rf}. The property states that there exists $p_* \in \Delta^M$ together with some region $\ccal A \subseteq \Delta^M$ such that for all $q \in \ccal A$, $\sigma_{\rho, \ccal S}(q) \geq \sigma_{\rho, \ccal S}(p_*)$. The bullet point represents ${\frak D}_{\ccal S}(p_*)$, and the shaded region represents ${\ccal T}_{\ccal S,p}^{(2)}(\{0,2\})$.  The various types of risk functionals  are classified in Figure~3 in the supplementary material. 
\begin{figure}
\begin{center}
	\begin{tikzpicture}[scale=1.8]
	\fill[color=lightgray] (1.23066,1.33523) -- (0.75,0.5) -- (0.51934,0.89952) -- (1,{sqrt(3)}) -- (1.23066,1.33523);
	\draw[dash pattern = on 2pt off 2pt] (0.46132,0) -- (1.23066,1.33523) node[anchor=south west] {\footnotesize $x_2 = p_2$};
	\draw[dash pattern = on 2pt off 2pt] (0.28868,0.5) node[anchor= east] {\footnotesize $x_1 = p_1$} -- (1.71132,.5);
	\draw[dash pattern = on 2pt off 2pt] (0.51934,0.89952) -- (1.03868,0) node[anchor=north] {\footnotesize $x_0 = p_0$};
	\draw[dash pattern = on 2pt off 2pt] (0.46132,0) -- (1.23066,1.33523);
	\draw (0,0) node[anchor=north east] {\footnotesize $(1,0,0)$} -- (2,0) node[anchor=north west] {\footnotesize $(0,0,1)$} -- (1,{sqrt(3)}) node[anchor=south] {\footnotesize $(0,1,0)$} -- (0,0); 
	\draw (0.75,0.5) node {\footnotesize $\bullet$};
	\draw (.875,1.11603) node {\footnotesize $\{0,2\}$};
	\draw (1.23066,0.77751) node {\footnotesize $\{0\}$};
	\draw (1.375,.25) node {\footnotesize $\{0,1\}$};
	\draw (.75,.16667) node {\footnotesize $\{1\}$};
	\draw (.375,.25) node {\footnotesize $\{1,2\}$};
	\draw (.51934,.63317) node {\footnotesize $\{2\}$};
\end{tikzpicture}
\end{center}
\caption{A illustration of a dominant $\rho$ (Definition~\ref{def: dominant_risk_functional}), where $x_i$'s denote the arguments of the Dirichlet distribution in~\eqref{eqn:dir}.}
\label{fig: dom_rf}
\end{figure}
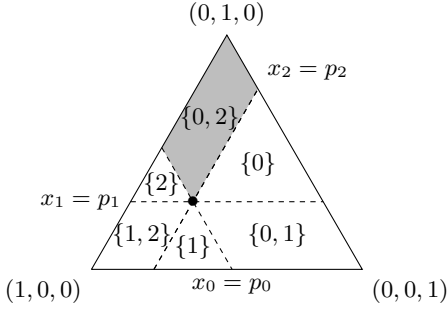
\subsubsection{Tail Lower Bound}
We show that dominant risk functionals $\rho$ satisfy a generalization of the tail lower bound developed by \citet{pmlr-v117-riou20a} and \citet{baudry2020thompson}.
\begin{lemma} \label{lem: lower_bound}
Fix a finite set $\ccal S$ with size $M+1$ and $r \in \RR$. Let $\rho : \ccal P \to \RR$ be a dominant risk functional such that $|\ccal I| = M$ (see Definition~\ref{def: dominant_risk_functional}). Fix $\alpha \in \NN^{M+1}$, $n= \sum_{i=0}^M \alpha_i$, and $p = \alpha/n$.
For $L \sim \Dir(\alpha)$ and large $n$,
	$$\PP(\sigma_{\rho,\ccal S}(L) \geq r) \geq C_2 n^{-\frac{M+1}{2}} \exp(-n {\ccal K}_{\inf}^{{\rho}_{\ccal S}}({\frak D}_{\ccal S}(p),r)),$$
	where $C_2 := \sqrt{2\pi}\paren{M/2.13}^{M/2}$.
\end{lemma}
We remark that Lemma~\ref{lem: lower_bound} generalises the lower bound in \citet[Lemma~2]{baudry2020thompson} to dominant risk functionals.
\subsubsection{Examples of Dominant Risk Functionals}
We provide numerous examples of risk functionals that satisfy the proposed notion of dominance.
\begin{proposition}\label{prop: table_of_dominant}
	Let $\rho:\ccal P \to \RR$ be a risk functional satisfying (\ref{eqn: infimum_condition}). Suppose $\rho$ satisfies one of following two properties: (a) $\rho$ is a distorted risk functional; (b) for any finite alphabet $\ccal S$ with size $M+1$, $\sigma_{\rho,\ccal S} = g|_{\Delta^M}$ for some convex $g : \RR^{M+1} \to \RR$ with first-order partial derivatives.
	Then $\rho$ is dominant.
\end{proposition}
\begin{corollary}\label{cor: cts_dominant_examples}
	The risk functionals in Table~\ref{table: common_drf} that are continuous (indicated by $\checkmark$) and those in Table~\ref{table: common_EDPM} that are convex (indicated by $\checkmark$) are   continuous and dominant. 
\end{corollary}
\section{Problem Formulation}
\label{problem}
Given a risk functional $\rho$ on a compact metric subspace $({\ccal C},d) \subset ({\ccal P},d)$ of probability measures and $K$ arms with probability measures ${(\nu_k)}_{k \in [K]} \subset {\ccal C}$, the learner's objective is to choose the \vocab{optimal arm} $k^* := \argmax_{k \in [K]} \rho(\nu_k)$ as many times as possible. All other arms $k \neq k^*$ are called \vocab{suboptimal}. Here we adopt the convention that the arm with higher $\rho(\nu_k)$ offers a higher reward. To adopt the cost perspective, consider the negation of the reward, and the objective as choosing the \emph{minimum}  $\rho(\nu_k)$ over all $k \in [K]$.

Akin to \citet{tamkin2020dist}, \citet{baudry2020thompson}, and \citet{chang2021riskconstrained}, we assess the performance of an algorithm $\pi$ using $\rho$, defined at time $n$, by the \vocab{$\rho$-risk regret}
\begin{align*}
{\ccal R}_\nu^{\rho}(\pi, n) &= \EE_\nu\parenb{\sum_{t=1}^n \paren{\max_{k \in [K]} \rho(\nu_k) - \rho(\nu_{A_t})}}\\
&= \EE_\nu \parenb{\sum_{t=1}^n \Delta_{A_t}^{\rho}} = \sum_{k=1}^K \EE_\nu[T_k(n)] \Delta_k^{\rho},
\end{align*}
where $\Delta_k^{\rho} := \rho(\nu_{k^*}) - \rho(\nu_k)$ is the difference between the expected reward of arm $k$ and that of the optimal arm $k^*$, and $T_k(n) = \sum_{t=1}^n \II (A_t=k)$ is the number of pulls of arm $k$ up to  and including time~$n$.

\section{Lower Bound}
\label{lower_bounds}
We establish an instance-dependent lower bound on the regret incurred by any \vocab{consistent} policy $\pi$, that is, $\lim_{n \to \infty} {\ccal R}_{\nu}^{\rho}(\pi,n)/n^a = 0$ for any $a > 0$.
\begin{theorem}\label{thm: rho_lower_bound}
Let ${\ccal Q} = {\ccal Q}_1 \times \dots \times {\ccal Q}_K$ be a set of bandit models $\nu = (\nu_1,\dots,\nu_K)$ where each $\nu_k$ belongs to the class of distributions ${\ccal Q}_k$. Let $\pi$ be any consistent policy. Suppose without loss of generality that $1$ is the optimal arm, i.e. $r_1^\rho = \max_{k \in [K]}r_k^\rho$. For any $\nu \in \ccal Q$ and suboptimal arm $k$,
$$\liminf_{n \to \infty} \frac{\EE_\nu[T_k(n)]}{\log n} \geq \frac{1}{{\ccal K}_{\inf}^{{\rho},{\ccal Q}_k} (\nu_k,r_1^\rho)}.$$
\end{theorem}
The proof follows that of \citet{baudry2020thompson} by replacing $(\mathrm{CVaR}_\alpha, c^*)$ therein by $(\rho, r_1^\rho)$, who in turn adapted the proof in \citet{doi:10.1287/moor.2017.0928} for their lower bound on the CVaR regret on consistent policies, and thus we relegate it to the supplementary material for brevity.

\section{The $\rho$-MTS and $\rho$-NPTS Algorithms}
\label{algorithms}
We design  two Thompson sampling-based algorithms, which follow in the spirit of \citet{pmlr-v117-riou20a} and \citet{baudry2020thompson}, called $\rho$-Multinomial-TS ($\rho$-MTS) (resp.\ $\rho$-Nonparametric-TS ($\rho$-NPTS)), where each $\nu_k$ follows a multinomial distribution (resp. distribution with bounded support).
\subsection{$\rho$-Multinomial-TS ($\rho$-MTS)}
Denote the Dirichlet distribution of parameters $\alpha = (\alpha^0,\alpha^1,\dots,\alpha^M)$ by $\Dir(\alpha)$ with density function 
\begin{equation}
f_{\Dir(\alpha)}(x) = \frac{\Gamma(\sum_{i=1}^n \alpha^i)}{\prod_{i=1}^n \Gamma(\alpha^i)} \prod_{i=1}^n x_i^{\alpha^i-1}, \label{eqn:dir}
\end{equation}
 where $x \in \Delta^{M} $.
The first algorithm, $\rho$-MTS, generalises the index policy in \citet{baudry2020thompson} from $\mathrm{CVaR}_\alpha$ to $\rho$.
The conjugate of the multinomial distribution is precisely the Dirichlet distribution.
Hence, we generate samples from the Dirichet distribution, and demonstrate that $\rho$-MTS is optimal in the case where for each $k \in [k]$, $\nu_k$ follows a multinomial distribution with support $\mathcal{S} =\{s_0,s_1,\dots,s_M\}$ regarded as a subset of $C$, $|\mathcal{S}| = M+1$, $s_0<s_1<\ldots<s_M$ without loss of generality, and probability vector $p_k \in \Delta^M$.
\begin{algorithm}[t]
   \caption{$\rho$-MTS}
   \label{alg: rho_mts}
\begin{algorithmic}[1]
\STATE \textbf{Input:} Continuous risk functional $\rho$, horizon $n$, support $S = \{s_0,s_1,\dots,s_M\}$.
\STATE Set $\alpha_k^m := 1$ for $k \in [K]$, $m \in {[M]}_0$, denote $\alpha_k = (\alpha_k^0,\alpha_k^1,\dots,\alpha_k^M)$.
\FOR {$t \in [n]$}
	\FOR {$k \in [K]$}
		\STATE Sample $L_k^t \sim \Dir(\alpha_k)$.
		\STATE Compute $r_{k,t}^{\rho} = \rho(\frak D_S(L_k^t))$.
	\ENDFOR
	\IF {$t \in [K]$}
		\STATE Choose action $A_t =t$.
	\ELSE
		\STATE Choose action $A_t =  \argmax_{k \in [K]} r_{k,t}^{\rho}$.
	\ENDIF
	\STATE Observe reward $X_{A_t}$.
	\STATE Increment $a_{A_t}^m$ by $\II\{X_{A_t} = s_m\}$, $m \in {[M]}_0$.
\ENDFOR
\end{algorithmic}
\end{algorithm}
In particular, for each $k \in [K]$, we initialise arm $k$ with a distribution of $\Dir(1^{M+1})$, the uniform distribution over $\Delta^M$, where for any $d \in \NN$, we denoted $1^d := (1,\dots,1) \in \RR^d$.
After $t$ rounds, the posterior distribution of arm $k$ is given by $\Dir(1+T_k^0(t), \dots, 1+T_k^{M}(t))$, where $T_k^i(t)$ denotes the number of times arm $k$ was chosen and reward $s_i$ was received until time $t$.
Let $\nu_k := {\frak D}_{\ccal S}(p_k)$ denote the distribution of arm $k$, where $p_k = (p_k^0,p_k^1,\dots,p_k^M) \in {\Delta}^{M}$.
\subsection{$\rho$-Nonparametric-TS ($\rho$-NPTS)}
To generalise to the bandit setting where the $K$ arms have general distributions with supports in $C \subseteq [0,1]$, we propose the $\rho$-NPTS algorithm.
Unlike $\rho$-MTS that samples for each $k \in [K]$ a probability distribution over a fixed support $\{s_0,s_1,\dots,s_M\} \subset C$, $\rho$-NPTS samples for each $k \in [K]$ a probability vector $L_k^t \Fol \Dir(1^{N_k})$ over $(1,X_1^k,\dots,X_{N_k}^k)$, where $N_k$ is the number of times arm $k$ has been pulled so far. 
\begin{algorithm}[t]
   \caption{$\rho$-NPTS}
   \label{alg: rho_npts}
\begin{algorithmic}[1]
\STATE \textbf{Input:} Continuous risk functional $\rho$, horizon $n$, history of the $k$-th arm $S_k = (1)$, $k \in [K]$.
\STATE Set $S_k := (1)$ for $k \in [K]$, $N_k = 1$.
\FOR {$t \in [n]$}
	\FOR {$k \in [K]$}
		\STATE Sample $L_k^t \sim \Dir(1^{N_k})$.
		\STATE Compute $r_{k,t}^{\rho} = \rho({\frak D}_{S_k}(L_k^t))$.
	\ENDFOR
	\STATE Choose action $A_t = \argmax_{k \in [K]} r_{k,t}^{\rho}$.
	\STATE Observe reward $X_{A_t}$.
	\STATE Increment $N_{A_t}$ and update $S_{A_t} := (S_{A_t}, X_{A_t})$.
\ENDFOR
\end{algorithmic}
\end{algorithm}
Thus, the support of the sampled distribution for $\rho$-NPTS depends on the observed reward, and is not technically a posterior sample with respect to some fixed prior distribution.

\section{Regret Analysis of $\rho$-MTS}
\label{regret_analyses}
In this section we present our regret guarantee for $\rho$-MTS, and show that it matches the lower bound in Theorem~\ref{thm: rho_lower_bound} and thus is \vocab{asymptotically optimal}.
\begin{theorem}\label{thm: rho_mts_upper_bound}
	Fix a finite alphabet $\ccal S  = \{s_0,s_1,\dots,s_M\} \subset C \subseteq [0,1]$.
	Let $\nu = {(\nu_k)}_{k \in [K]} \subset ({\ccal P}_{\ccal S},D_\infty)$ be a $K$-armed bandit model. Let $\rho$ be a continuous and dominant risk functional on $(\ccal P,D_\infty)$ such that $|\ccal I| = M$ (see Definition~\ref{def: dominant_risk_functional}). Then the regret of $\rho$-MTS is given by
	\begin{equation*}\label{eqn: almost-optimal}
	{\ccal R}_\nu^{\rho}(\rho\text{-MTS}, n) \leq \sum_{k : \Delta_k^{\rho} > 0} \frac{\Delta_k^{\rho} \log n}{{\ccal K}_{\inf}^{\rho}(\nu_k,r_1^{\rho})} + o(\log n),
	\end{equation*}
	where $r_k^{\rho} = \rho(\nu_k)$ for each $k \in [K]$, and $r_1^{\rho} = \max_{k \in [K]} r_k^{\rho}$ without loss of generality.
\end{theorem}
\begin{remark} \label{rmk:compare}
	{\em When $\rho = \EE[\cdot]$ and $\rho = {\mathrm{CVaR}}_\alpha$, $\rho $ satisfies $|\ccal I| = M$, and we recover the asymptotically optimal algorithms for multinomial distributions in \citet{pmlr-v117-riou20a} and \citet{baudry2020thompson} respectively.
	Furthermore, in the setting where $M=1$, $\rho = {\mathrm{MV}_\gamma}$  satisfies $|\ccal I| = M$, recovering the Bernoulli-MVTS algorithm proposed by \citet{zhu2020thompson}.
	Theorem~\ref{thm: rho_mts_upper_bound} improves their analysis significantly.
	First, we replace the term $\inv{(2\min\{{(p_1-p_i)}^2, {(1 - \gamma - p_1 - p_i)}^2\})}$ (where $\{p_i\}_{i=1}^K$ are the means of the Bernoulli distributions, and $\gamma$ is the risk tolerance of the mean-variance) that creates some slackness in their regret bound with the {\em exact} pre-constant ${\ccal K}_{\inf}^{\rho}(\nu_k,r_1)^{-1}$ in the $\log$ term.
	Second, we show this attains the  asymptotic lower bound in Theorem~\ref{thm: rho_lower_bound}.
	In general, the distorted risk functionals indicated by $\checkmark$ in Table~\ref{table: common_drf} are continuous and dominant for any $\ccal S$, and EPDMs in Table~\ref{table: common_EDPM} indicated by $\checkmark$ are continuous and dominant for $|\ccal S| = 2=:M+1$ (implying $|\ccal I|=1=M$), admitting asymptotically optimal $\rho$-MTS algorithms.}
\end{remark}

\begin{proof}[Proof Outline for Theorem~\ref{thm: rho_mts_upper_bound}]
	Fix $\eeps_1,\eeps_2>0$ and define the two events ${\ccal E}_1 := \{r_{k,t}^{\rho} \geq r_1^{\rho} - \eeps_1, D_\infty(\widehat{\nu}_k(t),\nu_k) \leq \eeps_2\}$ and ${\ccal E}_2 := {\ccal E}_1^c$,
where $(\widehat{\nu}_k(t),\nu_k) = ({\frak D}_{\ccal S}(\widehat{p}_k(t)),{\frak D}_{\ccal S}(p_k))$.
We upper bound $\EE[T_k(n)]$ by
\begin{align*}
\EE[T_k(n)] &\leq \underbrace{ \EE\!\parenb{\sum_{t = 1}^n \II(A_t=k, {\ccal E}_1)}}_{A} + \underbrace{\EE\!\parenb{\sum_{t = 1}^n \II(A_t=k, {\ccal E}_2)}}_B\\
&\leq \frac{\log n}{{\ccal K}_{\inf}^{\rho}(\nu_k, r_1^{\rho})-\eeps_3/2} + \ccal O(1),
\end{align*}
where $\eeps_3 \in (0,\min\{\eeps_1,\eeps_2\})$
by Lemmas~\ref{lem: upper_bounding_post_cv} and~\ref{lem: upper_bounding_pre_cv}, which are stated  below and proven in the supplementary material.
Taking $\eeps_1 \to 0^+,\eeps_2 \to 0^+$,
\begin{align*}
{\ccal R}_\nu^{\rho}(\text{$\rho$-MTS}, n)
&\leq \sum_{k : \Delta_k^{\rho} > 0} \frac{\Delta_k^{\rho} \log n}{{\ccal K}_{\inf}^{\rho}(\nu_k,r_1^{\rho})} + o(\log n).
\end{align*}
as desired.
\end{proof}
\begin{lemma}\label{lem: upper_bounding_post_cv}
	Suppose $\rho$ is continuous on $({\ccal P}_{\ccal S},D_\infty)$.
	For small $\eeps_1,\eeps_2>0$, $\eeps_3 \in (0,\max\{\eeps_1,\eeps_2\})$, 
	$$A \leq \frac{\log n}{{\ccal K}_{\inf}^{\rho}(\nu_k, r_1^{\rho})-\eeps_3/2} + \ccal O(1)\quad \mbox{as}\; \; n \to\infty.$$
\end{lemma}
\begin{lemma}\label{lem: upper_bounding_pre_cv}
	Suppose $\rho$ satisfies the conditions in Theorem~\ref{thm: rho_mts_upper_bound}.
	For small enough $\eeps>0$,
	$B \leq \ccal O(1)$ as $n\to\infty$.
\end{lemma}
These lemmas, which are proved in the supplementary material, arise from the upper bound for continuous risk functionals (Lemma~\ref{lem: upper_bound}) and the lower bound for dominant risk functionals (Lemma~\ref{lem: lower_bound}). These concentration bounds generalise the conclusions of \citet{pmlr-v117-riou20a} and \citet{baudry2020thompson} to continuous and dominant risk functionals, canonical examples include $\EE[\cdot]$ and $\mathrm{CVaR}_\alpha$.
\section{Numerical Experiments}
\label{numerical_experiments}
We verify our theory via numerical experiments on $\rho$-NPTS for new risk measures that are linear combinations of existing ones. These risk measures illustrate the generality and versatility of the theory developed.
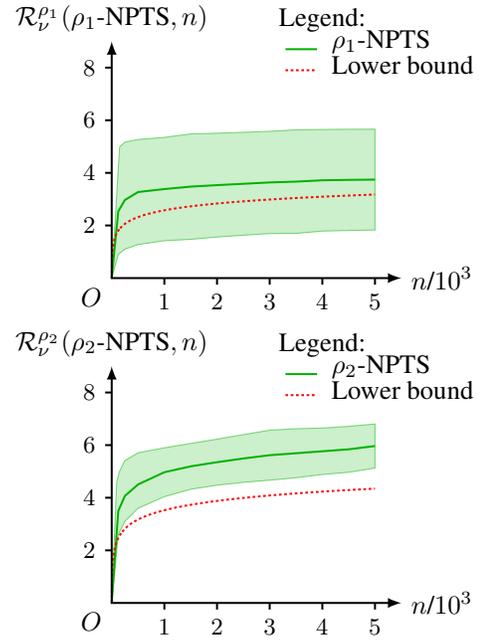
\begin{figure}[ht]
\centering
\begin{tikzpicture}[>=latex,scale=.7]
	\draw (3,8/2+.5) node[anchor=south west] {Legend:};
	\draw (4,7/2+.1+.5) node[anchor=south west] {$\rho_1$-NPTS};
	\draw[color=black!30!green, line width = .75pt] (3.3,7/2+.25+.1+.5) -- (3.9,7/2+.25+.1+.5);
	\draw (4,6/2+.2+.5) node[anchor=south west] {Lower bound};
	\draw[color=red, line width = .75pt, dash pattern = on 1pt off 1pt] (3.3,6/2+.25+.2+.5) -- (3.9,6/2+.25+.2+.5);
	
	\filldraw[color=black!30!green, fill opacity=0.2, draw opacity = 0.5] 
	(0,0/2)
	-- (.125,0.89953455464824/2)
	-- (.25,1.09953455464824/2)
	-- (.5,1.27415192312018/2)
	-- (1,1.41979512793847/2)
	-- (1.5,1.47330490434364/2)
	-- (2,1.55977244207752/2)
	-- (2.5,1.62766454980813/2)
	-- (3,1.69297962927133/2)
	-- (3.5,1.69845375192941/2)
	-- (4,1.78620046646474/2)
	-- (4.5,1.80920236155958/2)
	-- (5,1.82727252035333/2)
	-- (5,5.66043713265675/2)
	-- (4.5,5.65781813145049/2)
	-- (4,5.65001710768157/2)
	-- (3.5,5.63524130676181/2)
	-- (3,5.57864794941988/2)
	-- (2.5,5.54144051342800/2)
	-- (2,5.50681010570352/2)
	-- (1.5,5.48017972684609/2)
	-- (1,5.34841034779618/2)
	-- (0.5,5.26831936170429/2)
	-- (0.25,5.16831936170429/2)
	-- (0.15,4.98900660231708/2)
	-- (0,0/2);
	\draw[color=black!30!green, line width = .75pt] (0,0/2) --
	(0.125, 2.54/2)
	-- (0.25,2.9677553110588447/2) -- (.5,3.2712356424122317/2) -- (1,3.384102737867318/2) -- (1.5,3.4767423155948616/2) -- (2,3.533291273890519/2) -- (2.5,3.5845525316180624/2) -- (3,3.6358137893456046/2) -- (3.5,3.666847529345605/2) -- (4,3.718108787073147/2) -- (4.5,3.7335102465050327/2) -- (5,3.7438548265050327/2);
	\draw[color=red, samples=100, domain = 0:.15,line width = .75pt, dash pattern = on 1pt off 1pt] plot (\x,{8.39144115051507*ln(1000*\x+1)/45});
	\draw[color=red, samples=100, domain = .16:5,line width = .75pt, dash pattern = on 1pt off 1pt] plot (\x,{8.39144115051507*ln(1000*\x+1)/45});
	\draw[<->,line width = .75pt] (0,8/2+.5) node[anchor=south] {${\ccal R}_{\nu}^{\rho_1}(\text{$\rho_1$-NPTS}, n)$} -- (0,0/2) node[anchor=north east] {$O$} -- (5.5,0/2) node[anchor=west] {$n$/${10}^3$};
	\foreach \i in {1,2,...,5} {
		\draw[line width = .75pt] (\i,0) -- (\i,-.125) node[anchor=north]{\footnotesize $\i$};
	}
	\foreach \i in {2,4,6,8} {
		\draw[line width = .75pt] (-.125,\i/2) node[anchor=east]{\footnotesize $\i$} -- (0,\i/2) ;
	}
\end{tikzpicture}
\begin{tikzpicture}[>=latex,scale=.7]
	\draw (3,8/2+.5) node[anchor=south west] {Legend:};
	\draw (4,7/2+.1+.5) node[anchor=south west] {$\rho_2$-NPTS};
	\draw[color=black!30!green, line width = .75pt] (3.3,7/2+.25+.1+.5) -- (3.9,7/2+.25+.1+.5);
	\draw (4,6/2+.2+.5) node[anchor=south west] {Lower bound};
	\draw[color=red, line width = .75pt, dash pattern = on 1pt off 1pt] (3.3,6/2+.25+.2+.5) -- (3.9,6/2+.25+.2+.5);

	\filldraw[color=black!30!green, fill opacity=0.2, draw opacity = 0.5] 
	(0,0/2)
	-- (.125, 2.59953455464824/2)
	-- (.25,3.09953455464824/2)
	-- (.5,3.59953455464824/2)
	-- (1,4.03995591522903/2)
	-- (1.5,4.32234177463733/2)
	-- (2,4.47599816945615/2)
	-- (2.5,4.58493488805256/2)
	-- (3,4.66623537019264/2)
	-- (3.5,4.75971019973128/2)
	-- (4,4.88289605363840/2)
	-- (4.5,4.96700077021347/2)
	-- (5,5.12756567890668/2)
	-- (5,6.80206358219021/2)
	-- (4.5,6.71361125981083/2)
	-- (4,6.64844165179697/2)
	-- (3.5,6.62235318111516/2)
	-- (3,6.56655368606486/2)
	-- (2.5,6.39892194591649/2)
	-- (2,6.22562213420714/2)
	-- (1.5,6.06407772583945/2)
	-- (1,5.89821556756028/2)
	-- (0.5,5.70900660231708/2)
	-- (0.25,5.40900660231708/2)
	-- (0.15,5.00900660231708/2)
	-- (.095, 4.59953455464824/2)
	-- (0,0/2);
	\draw[color=black!30!green, line width = .75pt] (0,0) 
	-- (0.125, 3.5/2)
	-- (0.25,4.0677553110588447/2)
	-- (.5,4.5042705784826556/2)
	-- (1,4.969085741394648/2)
	-- (1.5,5.193209750238384/2)
	-- (2,5.350810151831642/2)
	-- (2.5,5.4919284169845195/2)
	-- (3,5.616394528128742/2)
	-- (3.5,5.691031690423209/2)
	-- (4,5.765668852717678/2)
	-- (4.5,5.8403060150121435/2)
	-- (5,5.964814630548434/2);
	\draw[color=red, samples=100, domain = 0:.2,line width = .75pt, dash pattern = on 1pt off 1pt] plot (\x,{6.888009706061486*ln(1000*\x+1)/27});
	\draw[color=red, samples=100, domain = 0.2:5,line width = .75pt, dash pattern = on 1pt off 1pt] plot (\x,{6.888009706061486*ln(1000*\x+1)/27});
	\draw[<->, line width = .75pt] (0,8/2+.5) node[anchor=south] {${\ccal R}_{\nu}^{\rho_2}(\text{$\rho_2$-NPTS}, n)$} -- (0,0/2) node[anchor=north east] {$O$} -- (5.5,0/2) node[anchor=west] {$n$/${10}^3$};
	\foreach \i in {1,2,...,5} {
		\draw[line width = .75pt] (\i,0) -- (\i,-.125) node[anchor=north]{\footnotesize $\i$};
	}
	\foreach \i in {2,4,6,8} {
		\draw[line width = .75pt] (-.125,\i/2) node[anchor=east]{\footnotesize $\i$} -- (0,\i/2) ;
	}
\end{tikzpicture}
\caption{Regrets with risks $\rho_1 = \mathrm{MV}_{0.5} + \mathrm{CVaR}_{0.95}$, $\rho_2 = \mathrm{Prop}_{0.7} + \mathrm{LB}_{0.6}$, and $n = 5,000$ over $50$ experiments.}
\label{fig: rho_1}
\end{figure}

We consider a $3$-arm bandit instance (i.e., $K=3$) with a horizon of $n=5,000$ time steps and over $50$  experiments, where the arms $1,2,3$ follow probability distributions $\mathrm{Beta}(1,3)$, $\mathrm{Beta}(3,3$), $\mathrm{Beta}(3,1)$ respectively. In particular, we have the means of each arm $i$ to equal $i/4$ for $i=1,2,3$. Define the risk functionals $\rho_1 := {\mathrm{MV}}_{0.5} + {\mathrm{CVaR}}_{0.95}$ and $\rho_2 := {\mathrm{Prop}}_{0.7} + {\mathrm{LB}}_{0.6}$ on $({\ccal P}_{\mathrm{c}}^{(B)},D_{\mathrm{L}})$, where we set $(\gamma,\alpha,p,q) = (0.5,0.95,0.7,0.6)$ as the parameters for the mean-variance, CVaR, Proportional risk hazard, and Lookback components respectively (see Table~\ref{table: common_drf}). By Example~\ref{eg: cty_of_linear_combi}, $\rho_j$ for $j=1,2$ are both continuous on $({\ccal P}_{\mathrm{c}}^{(B)},D_{\mathrm{L}})$. In Figure~\ref{fig: rho_1}, we plot the average empirical performance of $\rho_j$ respectively in green, together with their error bars denoting $1$ standard deviation. In both figures, we also plot the theoretical lower bound $\ell_{\rho_j}(n):=\sum_{k=1}^K (  \Delta_k^{\rho_j} \log n)/{\ccal K}_{\inf}^{\rho_j} (\nu_k,r_1^{{\rho_j}})$ (cf.\ Theorem~\ref{thm: rho_lower_bound}) in red and demonstrate that the regrets incurred by $\rho_j$-NPTS are competitive compared to the lower bounds, i.e., ${{\ccal R}_{\nu}^{\rho_j}(\text{$\rho_j$-NPTS},n)} \approx \ell_{\rho_j}(n)$ for $j = 1,2$ and large $n$.
The Java code to reproduce  the plots in Figure \ref{fig: rho_1} can be found at
tinyurl.com/unifyRhoTs.
\section{Conclusion}
\label{conclusion}
We posit the first unifying theory for Thompson sampling algorithms on risk-averse MABs.
We designed two Thompson sampling-based algorithms given any continuous and dominant risk functional, and prove for many of them asymptotic optimality in the case of multinomially and Bernoulli distributed bandits.
We generalise concentration bounds that utilise the continuity and dominance of the risk functional rather than its other properties.
There can be further analysis of $\rho$-NPTS, which we believe also solves the $\rho$-MAB, for $\rho$'s under appropriate conditions, with asymptotic optimality.
Further work can also adapt the techniques in \citet{baudry2021optimality}, who designed asymptotically optimal \vocab{Dirichlet sampling} algorithms for bandits whose rewards are unbounded but satisfy mild light-tailed conditions, to the risk-averse setting.

\section{Acknowledgements}
The authors would like to thank  the   reviewers of AAAI 2022 for their detailed and constructive comments, Qiuyu Zhu for initial discussions, and Emilie Kaufmann and Dorian Baudry for indispensable feedback on an earlier version of the manuscript. This research/project is supported by the National Research Foundation, Singapore under its AI Singapore Programme (AISG Award No: AISG2-RP-2020-018) and Singapore Ministry of Education AcRF Tier 1 grant (R-263-000-E80-114).
\newpage
\bibliographystyle{aaai22}
\bibliography{Distorted_Risk_TS}
\appendix

\onecolumn
\begin{center}
\textbf{\large Supplementary Material}
\end{center}

\section{Well-Definedness of Definition~\ref{def: risk functional}}
We recall that a \vocab{risk functional} is an $\RR$-valued mapping $\rho : {\ccal P} \to \RR$ on ${\ccal P}$. Let ${\frak R}$ denote the collection of risk functionals on ${\ccal P}$. A \vocab{conventional risk functional} $\varrho : {\ccal L}_\infty \to \RR$ is an $\RR$-valued mapping on ${\ccal L}_{\infty}$.
Let $X \sim \mu$ denote the random variable $X$ sampled from the probability measure $\mu$.
A conventional risk functional $\varrho : {\ccal L}_\infty \to \RR$ is said to be \vocab{law-invariant} \citep{huang2021offpolicy} if for any pair of $C$-valued random variables $X_i : (\Omega_i,{\ccal F}_i, \PP_i) \to (C,{\frak B}(C))$ sampled from probability measures $\mu_{i} := \PP_i \circ \inv{X_i} \in {\ccal P}$, $i=1,2$, $$\mu_{1} = \mu_{2} \To \varrho(X_1) = \varrho(X_2).$$
Let ${\frak L}$ denote the set of law-invariant conventional risk functionals. We would like to `translate' between the existing progress worked on law-invariant conventional risk functionals $\varrho : {\ccal L}_{\infty} \to \RR$ on the space of  bounded random variables into progress of risk functionals $\rho : {\ccal P} \to \RR$ on the space of probability measures on $C$. Succinctly, we would like to construct a well-defined bijection from ${\frak R}$ to ${\frak L}$.
Indeed, for any $\rho \in \frak R$, declare $\Phi(\rho) \in \frak L $ by $\Phi(\rho)(X) = \rho(\mu)$, where $X \sim \mu$. This is well defined, since for any $X,X' \sim \mu$, $$\Phi(\rho)(X) = \rho(\mu) = \Phi(\rho)(X').$$
Furthermore, for $\rho = \rho'$ and $X \sim \mu$, $$\Phi(\rho)(X) = \rho(\mu) = \rho'(\mu) = \Phi(\rho')(X) \To \Phi(\rho) = \Phi(\rho').$$
Hence, the mapping $\Phi: \frak R \to \frak L, \rho \mapsto \Phi(\rho)$ is well defined. Similarly define $\Psi(\varrho) \in \frak R$ by $\Psi(\varrho)(\mu) = \varrho(X)$, and the well-defined mapping $\Psi : \frak L \to \frak R, \varrho \mapsto \Psi(\varrho)$. Then for any $\rho \in \frak R$, $\mu \in \ccal P$, and $X \sim \mu$,
$$(\Psi \circ \Phi)(\rho)(\mu) =\Psi(\Phi(\rho))(\mu) = \Phi(\rho)(X) = \rho(\mu),$$
thus $(\Psi \circ \Phi)(\rho) = \rho \To \Psi \circ \Phi = \mathrm{id}_{\frak R}$. Similarly, $\Phi \circ \Psi = \mathrm{id}_{\frak L}$, and we have the required bijection $\Phi : \frak R \to \frak L$.

\section{Proofs of Properties of Continuous Risk Functionals}
We first state and prove some continuity properties of ${\ccal K}_{\inf}^\rho$, beginning with Lemma~\ref{lem: KL_lowersemicty} whose proof we omit for brevity.
\begin{lemma}[\citet{1055416}]\label{lem: KL_lowersemicty}
	The function $\KL(\cdot,\cdot) : {\ccal P}^2 \to \RR$ is lower-semicontinuous in both of its arguments in either setting $(\ccal C,d) = ({\ccal P}_{\ccal S},D_\infty),({\ccal P},D_{\mathrm{L}})$.
\end{lemma}
\begin{lemma}\label{lem: lower-semicty_KLinf_first_arg}
	For any fixed $r$, the mapping ${\ccal K}_{\inf}^\rho(\cdot,r) : \ccal P \to \RR$ is lower-semicontinuous.
\end{lemma}
\begin{proof}[Proof of Lemma~\ref{lem: lower-semicty_KLinf_first_arg}]
	Fix $r \in \RR$ and $y < {\ccal K}_{\inf}^\rho(\mu,r)$. By the definition of lower-semicontinuity, we want to find an open set $V \ni \mu$ such that for any $\varphi \in V$,
	$$y < {\ccal K}_{\inf}^\rho(\varphi,r).$$
	Fix $\eeps > 0$ such that $y < {\ccal K}_{\inf}^\rho(\mu,r) - \eeps$ and $\eta \in \ccal P$ with the property $\rho(\eta) \geq r$. By unraveling the definition of ${\ccal K}_{\inf}^\rho(\mu,r)$,
	$$y \leq \KL(\mu,\eta) - \eeps < \KL(\mu,\eta) - \frac{\eeps}{2}.$$
	By Lemma~\ref{lem: KL_lowersemicty}, $\KL(\cdot,\eta)$ is lower-semicontinuous at $\mu \in \ccal P$. Hence, there exists an open set $V_\eta \ni \mu$ such that for any $\gamma \in V_\eta$,
	$$y<\KL(\gamma,\eta)-\frac{\eeps}{2}.$$
	Define the open set $V = \bigcup_{\eta \in \ccal P}V_{\eta}$ and fix $\varphi \in V$. We have $\varphi \in V_{\eta_*}$ for some $\eta_* \in \ccal P$, and
	$$y < \KL(\varphi, {\eta_*}) - \frac{\eeps}{2}.$$
	Taking infimum over $\eta_* \in \ccal P$, $\rho(\eta_*) \geq r$,
	$$y \leq \inf_{\eta_* \in \ccal P}\{ \KL(\varphi, {\eta_*}) : \rho(\eta_*) \geq r\}- \frac{\eeps}{2} = {\ccal K}_{\inf}^\rho(\varphi,r)- \frac{\eeps}{2} < {\ccal K}_{\inf}^\rho(\varphi,r).$$
\end{proof}
\begin{lemma}\label{lem: lower-semicty_KLinf_second_arg}
	For any fixed $\mu$, the mapping $\KL_\mu^*(r ) = \inf_{\rho(\eta)= r } \KL(\mu,\eta)$ is lower-semicontinuous in $r$.
\end{lemma}
\begin{proof}[Proof of Lemma~\ref{lem: lower-semicty_KLinf_second_arg}]
	By \citet[Contraction Principle]{dembo2009large}, $\KL_\mu^*$ is a rate function of a suitably chosen sequence of measures that satisfies a large deviations principle (LDP), and is thus lower semicontinuous.
\end{proof}
\begin{lemma}\label{lem: lower-semicty_KLinf}
	Let $K \subset \RR$ be a compact set. Let $h : K \to \RR$ be a lower-semicontinuous function. Then the function $g : K \to \RR$ defined by
	$g(r) = \inf_{r' \in [r,\infty) \cap K} h(r')$ is lower-semicontinuous.
\end{lemma}
\begin{proof}[Proof of Lemma~\ref{lem: lower-semicty_KLinf}]
	Fix $a \in \RR$. It suffices to prove that $A := \{x \in K : g(x) \leq a\}$ is closed. Let $(s_n)$ be a sequence in $A$ such that $s_n \to s$ for some $s \in K$. By the Extreme Value Theorem, there exists $t_n \in [s_n,\infty) \cap K$ such that $g(s_n) = h(t_n)$. Since $[r,\infty) \cap K$ is compact, there exists a convergent subsequence $(t_{n_k})$ of $(t_n)$ such that $t_{n_k} \to t$ for some $t \in K$. Since $h$ is lower-semicontinuous and $h(t_{n_k}) =g(s_{n_k}) \leq a$, we have $h(t) \leq a$.  Since $t_{n_k} \to t$ and $t_{n_k} \geq s_{n_k}$ for each $k$, $t = \lim_{k \to \infty} t_{n_k} \geq \lim_{k \to \infty} s_{n_k} = s$. By the definition of $g$, $g(s) \leq h(r')$ for any $r' \geq s$. In particular, $g(s) \leq h(t) \leq a$. Thus, $s \in A$, and $A$ is closed, as required.
\end{proof}
\begin{proof}[Proof of Proposition~\ref{cor: continuity_of_KL}]
	By Lemma~\ref{lem: lower-semicty_KLinf_first_arg} ${\ccal K}_{\inf}^{\rho}$ is lower-semicontinuous on its first argument. By Lemma~\ref{lem: lower-semicty_KLinf_second_arg}, $\KL_\mu^*( r) = \inf_{\rho(\eta)=r} \KL(\mu,\eta)$ is lower-semicontinuous in $r$. By Lemma~\ref{lem: lower-semicty_KLinf}, $\ccal K_{\inf}^\rho(\mu, r) = \inf_{r' \in [r ,\infty) \cap \rho(\ccal C)} \KL_\mu^*(r')$ is lower-semicontinuous in $r$ on the compact set $\rho(\ccal C)\subset \RR$.
\end{proof}
\begin{proposition}\label{prop: continuity_props}
	Let $\rho : \ccal P \to \RR$ be a risk functional continuous on $({\ccal P}_{\ccal S},D_\infty)$ for some finite alphabet $\ccal S$. The mapping ${\ccal K}_{\inf}^{\rho_{\ccal S}}$ is non-decreasing in its second argument. Furthermore, ${\ccal K}_{\inf}^{\rho_{\ccal S}}({\frak D}_{\ccal S}(p),\cdot)$ is strictly increasing on $(\sigma_{\rho,\ccal S}(p),s_M]$.
\end{proposition}
\begin{proof}[Proof of Proposition~\ref{prop: continuity_props}]
Fix $p \in \Delta^M$. Then for any $r \leq s$,
$${\ccal T}_{\rho,\ccal S}(r) \supseteq {\ccal T}_{\rho,\ccal S}(s) \quad \To \quad {\ccal K}_{\inf}^{\rho_{\ccal S}}({\frak D}_{\ccal S}(p),r) \leq {\ccal K}_{\inf}^{\rho_{\ccal S}}({\frak D}_{\ccal S}(p),s),$$
and ${\ccal K}_{\inf}^{\rho_{\ccal S}}$ is non-decreasing in its second argument. To prove that  ${\ccal K}_{\inf}^{\rho_{\ccal S}}({\frak D}_{\ccal S}(p),\cdot)$ is strictly increasing on $(\sigma_{\rho,\ccal S}(p),s_M]$, we follow the argument from \citet{baudry2020thompson}. It suffices to prove that the constraints are binding in the optimum. Suppose otherwise, that there exists $r \in (\sigma_{\ccal S,\rho}(p),s_M)$ and $p_r^* \in \Delta^M$ such that
$${\ccal K}_{\inf}^{\rho_{\ccal S}}({\frak D}_{\ccal S}(p),r)= \KL({\frak D}_{\ccal S}(p),{\frak D}_{\ccal S}(p_r^*)) \quad \mbox{and} \quad \exists \delta > 0 : \sigma_{\rho,\ccal S}(p_r^*) = r+\delta.$$
By the continuity of $\sigma_{\rho,\ccal S}$, there exists $\eeps > 0$ such that for any $q \in \Delta^M$ with $d_\infty(q,p) < \eeps$,
$$|\sigma_{\rho,\ccal S}(p_r^*)-\sigma_{\rho,\ccal S}(p_r^*)|<\delta/2 \quad \To \quad \sigma_{\rho,\ccal S}(p_r^*) > r+\delta/2 > r$$
and
$$\KL({\frak D}_{\ccal S}(p),{\frak D}_{\ccal S}(q))\geq \KL({\frak D}_{\ccal S}(p),{\frak D}_{\ccal S}(p_r^*)).$$
By the arguments in \citet{baudry2020thompson}, this implies a contradiction.
\end{proof}
We next prove some of the bounds which are essential in the analysis of the proof of Theorem~\ref{thm: rho_mts_upper_bound}.
\begin{proposition}\label{prop: continuity_rho}
	Let $\rho : \ccal P \to \RR$ be a risk functional continuous on $({\ccal P}_{\ccal S},D_\infty)$ for some finite alphabet $\ccal S$. Fix $\alpha \in \NN^{M+1}, n = \sum_{i=0}^M \alpha_i$. Then there exists $N \in \NN$ such that for $n \geq N$,
	$$\parenl{\sigma_{\rho,\ccal S}\paren{\frac{\alpha}{n+M+1}}-\sigma_{\rho,\ccal S}\paren{\frac{\alpha-1}{n}}} < \eeps.$$
\end{proposition}
\begin{proof}[Proof of Proposition~\ref{prop: continuity_rho}]
	Since $\rho$ is continuous, $\sigma_{\rho,\ccal S}$ is uniformly continuous on $\Delta^M$, and there exists $\delta > 0$ such that for any $p,q \in \Delta^M$,
	$$d_\infty(p,q) < \delta \quad \To\quad |\sigma_{\rho,\ccal S}(p)-\sigma_{\rho,\ccal S}(q)|<\eeps.$$
	Choose $N$ such that $M/(N+M+1)<\delta$. Then for $n \geq N$,
	$$d_\infty\paren{\frac{\alpha}{n+M+1},\frac{\alpha-1}{n}} = \max_{i=0,\dots,M} \parenl{\frac{\alpha_i}{n+M+1} - \frac{\alpha_i-1}{n}} \leq \frac{M}{n+M+1} < \delta.$$
	Hence,
	$$\parenl{\sigma_{\rho,\ccal S}\paren{\frac{\alpha}{n+M+1}}-\sigma_{\rho,\ccal S}\paren{\frac{\alpha-1}{n}}}<\eeps.$$
\end{proof}
\begin{lemma}[DKW Inequality]\label{lem: dkw_inequality}
	Let $\rho: \ccal P \to \RR$ be a risk functional that is uniformly continuous on $\ccal Q \subseteq \ccal P$. Let $\seq{X_i}_{i=1}^\infty$ be a sequence of i.i.d. samples distributed according to some fixed $\mu \in \ccal Q$. Let $\widehat{\mu}_n := \frac{1}{n} \sum_{i=1}^n \delta_{X_i}$ denote the empirical distribution of the first $n$ samples. Then for any $\eeps > 0$, there exists $\delta_\eeps > 0$ such that
	$$\PP\paren{|\rho(\widehat{\mu}_n)-\rho(\mu)|>\eeps}\leq 2e^{-n\delta_\eeps^2/2}.$$
\end{lemma}
\begin{proof}[Proof of Lemma~\ref{lem: dkw_inequality}]
	Fix $\eeps > 0$. By the uniform continuity of $\rho$ on $\ccal Q$, there exists $\delta_\eeps > 0$ such that
	$$D_\infty(\mu,\eta)\leq \delta_\eeps/2 \quad \To \quad |\rho(\mu)-\rho(\eta)| < \eeps.$$
	Since $$\{D_\infty(\widehat \mu_n,\mu) \leq \delta_\eeps/2\} \subseteq \{|\rho(\widehat \mu_n) - \rho(\mu)| < \eeps\} \quad \iff \quad \{D_\infty(\widehat \mu_n,\mu) > \delta_\eeps/2\} \supseteq \{|\rho(\widehat \mu_n) - \rho(\mu)| \geq \eeps\},$$ by the DKW inequality, we have
	$$\PP(|\rho(\widehat \mu_n) - \rho(\mu)| > \eeps) \leq \PP(D_\infty(\widehat{\mu}_n,\mu) > \delta_{\eeps}/2) \leq 2e^{-n\delta_{\eeps}^2/2}.$$
\end{proof}
\begin{corollary}\label{cor: dkw_corollary}
	The conclusion of the previous lemma holds if $\rho : \ccal P \to \RR$ is continuous and $\ccal Q$ is compact. In particular, it is true for $\ccal Q={\ccal P}_{\ccal S}$ if $\rho$ is continuous with respect to $D_\infty$.
\end{corollary}
\begin{proof}[Proof of Corollary~\ref{cor: dkw_corollary}]
	Since $\rho$ is continuous on a compact set ${\ccal P}_{\ccal S}$, it is uniformly continuous on $\ccal Q={\ccal P}_{\ccal S}$.
\end{proof}
\begin{proof}[Proof of Lemma~\ref{lem: upper_bound}]
	By the lower-semicontinuity of $\KL(\cdot,\cdot)$ in both its arguments on a compact set, there exists a probability vector $p_* \in \Delta^M$ such that
	$$\KL({\frak D}_{\ccal S}(p),{\frak D}_{\ccal S}(p_*)) = {\ccal K}_{\inf}^{{\rho}_{\ccal S}}({\frak D}_{\ccal S}(p),r).$$
	A careful study of Lemma~13 in \citet{pmlr-v117-riou20a} reveals that convexity is not used in the proof. Hence, we can apply the result therein to the closed (and thus compact) sets $\inv{\sigma_{\rho,\ccal S}}([r,\infty))$ and $\inv{\sigma_{\rho,\ccal S}}((-\infty,r])$.
\end{proof}
\begin{proof}[Proof of Proposition~\ref{prop: distorted_risk_cty}]
	Fix $\eeps > 0$. By the continuity of $g$ on $[0,1]$, $g$ is uniformly continuous on $[0,1]$. Hence, there exists $\delta > 0$ such that for any $x,y \in [0,1]$,
	\begin{equation}\label{eqn: cty_of_g}
	\parenl{x-y} < \delta \To \parenl{g(x)-g(y)} < \frac{\eeps}{D}.
	\end{equation}
	For any $\mu,\eta \in {\ccal P}$ with CDFs $F_\mu, F_{\eta}$ respectively, suppose $\norm{F_\mu-F_{\eta}}_\infty = D_\infty(\mu,\eta) < \delta$. Then for any $t \in [0,D]$,
	\begin{align*}
	\parenl{1 - F_\mu(t) - (1 - F_{\eta}(t))}
	&\leq \norm{(1-F_\mu)-(1-F_{\eta})}_\infty = \norm{F_\mu-F_{\eta}}_\infty < \delta.
	\end{align*}
	By (\ref{eqn: distorted_risk_functional_bounded}) and (\ref{eqn: cty_of_g}),
	\begin{align*}
	\parenl{\rho_g(\mu) - \rho_g(\eta)}
	&= \parenl{\rint{0}{D}{g(1-F_\mu(t))}{t} - \rint{0}{D}{g(1-F_{\eta}(t))}{t}}\\
	&\leq \rint{0}{D}{\parenl{g(1-F_\mu(t)) - g(1-F_{\eta}(t))}}{t} \\
	&\leq \rint{0}{D}{\frac{\eeps}{D}}{t} \\
	&= \eeps,
	\end{align*}
	and $\rho_g$ is (uniformly) continuous on $({\ccal P},D_\infty)$.
\end{proof}
\section{Proofs of Properties of Dominant Risk Functionals}
\begin{proof}[Proof of Lemma~\ref{lem: lower_bound}]
	By the definition of dominant risk functionals, there exists $\mu_* = {\frak D}_{\ccal S}(p_*)$ with
	 $${\ccal K}_{\inf}^{\rho}(\mu,r) = \KL(\mu,\mu_*)$$
	 and $\ccal I \subseteq [M], \ccal I \neq \emptyset, {[M]}_0$, possibly depending on $p_*$, such that $${\ccal T}_{\rho,\ccal S, p_*}^{(1)} \supseteq {\ccal T}_{\ccal S, p_*}^{(2)}(\ccal I).$$
	 Thus,
	 $$\PP(\sigma_{\rho,\ccal S}(L) \geq r) \geq \PP(L \in {\ccal T}_{\rho,\ccal S,p_*}^{(1)}) \geq \PP(L \in {\ccal T}_{\ccal S,p_*}^{(2)}(\ccal I)).$$
	 Suppose furthermore that $|\ccal I| = M$. Let $j \notin \ccal I$.
	 Following the arguments in \citet{pmlr-v117-riou20a} and \citet{baudry2020thompson},
	 we can lower-bound the right-hand side by
	 $$\PP(L \in {\ccal T}_{\ccal S,p_*}^{(2)}(\ccal I))  \geq \frac{\Gamma(n)}{\prod_{i=0}^M \Gamma(\alpha_i)} \frac{\alpha_j}{p_j^*} \prod_{i \in \ccal I} \frac{{(p_i^*)}^{\alpha_i}}{\alpha_i} .$$
	 By the arguments in Lemma~2 of \citet{baudry2020thompson}, we can lower-bound the right-hand side by
	 $$\PP(L \in {\ccal T}_{\ccal S,p_*}^{(2)}(\ccal I)) \geq C_2 n^{-(M+1)/2} \exp(-n {\ccal K}_{\inf}^{{\rho}_{\ccal S}}({\frak D}_{\ccal S}(p),r)).$$
\end{proof}
\begin{proposition}\label{prop: drf_formula}
	For $\mu \in {\ccal P}_{\ccal S}$, let $\mu = {\frak D}_{\ccal S}(p)$ for some $p \in \Delta^M$. Then
	$$\rho_g(\mu) = \sum_{j=0}^M g\paren{\sum_{i=j}^M p_i} \Delta s_j,$$
	where $s_{-1} := 0$ for convenience, and $\Delta s_j = s_j - s_{j-1}$ for $j=0,\dots,M$.
\end{proposition}
\begin{proof}[Proof of Proposition~\ref{prop: drf_formula}]
	Consider,
	\begin{align*}
		\rho_g({\frak D}_{\ccal S}(p))
		&= \rint{s_{-1}}{s_M}{g(1-F_\mu(t))}{t}
		= \rint{s_{-1}}{s_M}{g\paren{\sum_{i=0}^M p_i \II\{s_i > t\}}}{t}\\
		&= \sum_{j=0}^M\rint{s_{j-1}}{s_j}{g\paren{\sum_{i=0}^M p_i \II\{s_i > t\}}}{t}
		= \sum_{j=0}^M\rint{s_{j-1}}{s_j}{g\paren{\sum_{i=0}^M p_i \II\{s_i \geq  s_j\}}}{t}\\
		&= \sum_{j=0}^M\rint{s_{j-1}}{s_j}{g\paren{\sum_{i=j}^M p_i}}{t}
		= \sum_{j=0}^Mg\paren{\sum_{i=j}^M p_i} (s_j-s_{j-1})
		= \sum_{j=0}^M g\paren{\sum_{i=j}^M p_i} \Delta s_j.
	\end{align*}
\end{proof}
\begin{proposition}\label{prop: drf_dominant_meta}
	Let $\rho_g$ be a distorted risk functional and $p \in \Delta^M$. Then ${\ccal T}_{\rho,\ccal S, p}^{(1)} \supseteq {\ccal T}_{\ccal S, p}^{(2)}({[M-1]}_0)$.
\end{proposition}
\begin{proof}[Proof of Proposition~\ref{prop: drf_dominant_meta}]
	Fix any $\mu = {\frak D}_{\ccal S}(q) \in {\ccal T}_{\ccal S, p}^{(2)}({[M-1]}_0)$. 
	We observe that for any $q \in {\ccal T}_{\ccal S, p}^{(2)}({[M-1]}_0)$ and $j \in {[M]}_0$, for $i=0,\dots,j-1$
	$$q_i \leq p_i\ \mbox{for $i\in {[j-1]}_0$} \quad \To \quad \sum_{i=0}^{j-1} q_i \leq \sum_{i=0}^{j-1} p_i.$$
	Hence,
	$$\sum_{i=j}^M q_i = 1 - \sum_{i=0}^{j-1} q_i \geq 1 - \sum_{i=0}^{j-1} p_i = \sum_{i=j}^M p_i.$$ By the monotonicity of $g$ and Proposition~\ref{prop: drf_formula},
	$$\rho_g({\frak D}_{\ccal S}(q)) = \sum_{j=0}^M g\paren{\sum_{i=j}^M q_i} \Delta s_j \geq \sum_{j=0}^M g\paren{\sum_{i=j}^M p_i} \Delta s_j = \rho_g({\frak D}_{\ccal S}(p)).$$
	Hence, ${\frak D}_{\ccal S}(q) \in {\ccal T}_{\rho,\ccal S, \mu_*}^{(1)}$, and we have ${\ccal T}_{\rho,\ccal S, \mu_*}^{(1)}\supseteq {\ccal T}_{\ccal S, \mu_*}^{(2)}({[M-1]}_0)$.
\end{proof}
\begin{corollary}\label{cor: drf_dominant}
	Let $\rho_g$ be a distorted risk functional with the property that there exists $\mu_* \in {\ccal T}_{\rho}(r)$ with
	 $${\ccal K}_{\inf}^{\rho}(\mu,r) = \KL(\mu,\mu_*).$$ Then $\rho_g$ is dominant.
\end{corollary}
\begin{proof}[Proof of Corollary~\ref{cor: drf_dominant}]
	The result follows from Proposition~\ref{prop: drf_dominant_meta} by setting $\ccal I={[M-1]}_0$.
\end{proof}
\begin{corollary}\label{cor: cts_drf_dominant}
	Let $\rho_g$ be a continuous distorted risk functional. Then $\rho_g$ is continuous and dominant.
\end{corollary}
\begin{proof}[Proof of Corollary~\ref{cor: cts_drf_dominant}]
	Since $\rho_g$ is continuous, there exists $\mu_* \in {\ccal T}_{\rho}(r)$ with
	 $${\ccal K}_{\inf}^{\rho}(\mu,r) = \KL(\mu,\mu_*).$$
	 The result follows from Corollary~\ref{cor: drf_dominant}.
\end{proof}
\begin{corollary}\label{cor: cts_df_dominant}
	Let $g$ be a continuous distortion function. Then the distortion risk functional $\rho_g$ is continuous and dominant.
\end{corollary}
\begin{proof}[Proof of Corollary~\ref{cor: cts_df_dominant}]
	By Proposition~\ref{prop: distorted_risk_cty}, $\rho_g$ is a continuous distorted risk functional. The result follows from Corollary~\ref{cor: cts_drf_dominant}.
\end{proof}
\begin{proposition}\label{prop: convex_risk_functional}
Let $\rho$ be a risk functional. Suppose there exists a function $h_{\rho,\ccal S} : \RR^M \to \RR$ with non-decreasing first-order partial derivatives such that
$$h_{\rho,\ccal S}(x) := \sigma_{\rho,\ccal S} \paren{ 1-\sum_{i=1}^M x_i, x} \quad \mbox{whenever} \quad \paren{ 1-\sum_{i=1}^M x_i, x} \in \Delta^M.$$
Then there exists $\ccal I \subseteq {[M]}_0, \ccal I \neq \emptyset, {[M]}_0$ depending on $p$ such that ${\ccal T}_{\rho,\ccal S, p}^{(1)} \supseteq {\ccal T}_{\ccal S, p}^{(2)}(\ccal I)$.
\end{proposition}
\begin{proof}[Proof of Proposition~\ref{prop: convex_risk_functional}]
	For any $p \in \Delta^M$, denote $p^+ = (p_1,\dots,p_M)$ and
	define $\ccal I = {\ccal I}_p$ depending on $p$ by
	$${\ccal I}_p := \begin{cases} \left\{i \in [M] : \frac{\del h_{\rho,\ccal S}}{\del x_i}(p^+) < 0 \right\} & \mbox{if $\exists i \in [M]$ such that $\frac{\del h_{\rho,\ccal S}}{\del x_i}(p^+) < 0$,}\\
	\{0\} & \mbox{otherwise.} \end{cases}$$
	We now prove that ${\ccal T}_{\rho,\ccal S, p}^{(1)} \supseteq {\ccal T}_{\ccal S, p}^{(2)}({\ccal I})$. Take ${\frak D}_{\ccal S}(q) \in {\ccal T}_{\ccal S, p}^{(2)}({\ccal I})$. We note that $\sigma_{\rho,\ccal S}(q) = h_{\rho,\ccal S}(q^+)$.
	By Lemma~\ref{lem: convex_lemma} below, $$\sigma_{\rho,\ccal S}(q) = h_{\rho,\ccal S}(q^+) \geq h_{\rho,\ccal S}(p^+) = \sigma_{\rho,\ccal S}(p) \quad \To \quad {\ccal T}_{\rho,\ccal S, p}^{(1)} \supseteq {\ccal T}_{\ccal S, p}^{(2)}(\ccal I).$$
\end{proof}
\begin{lemma}\label{lem: convex_lemma}
	Following the notation in Proposition~\ref{prop: convex_risk_functional},
	we have for any $q \in {\ccal T}_{\ccal S, p}^{(2)}({\ccal I}_p)$,
	$$h_{\rho,\ccal S}(q^+) \geq h_{\rho,\ccal S}(p^+).$$
\end{lemma}
\begin{proof}[Proof of Lemma~\ref{lem: convex_lemma}]
	If $1 \in \ccal I_p$, then $q_1 \leq p_1$ and by the Mean Value Theorem, there exists $c_1 \in [q_1,p_1]$ such that
	$$h_{\rho,\ccal S}(p_1,\dots,p_M) - h_{\rho,\ccal S}(q_1,\dots,p_M) = \frac{\del h_{\rho,\ccal S}}{\del x_1} (c_1,\dots,p_M) \leq \frac{\del h_{\rho,\ccal S}}{\del x_1} (p_1,\dots,p_M) < 0,$$
	where the second last inequality holds because the partial derivatives $\frac{\del h_{\rho,\ccal S}}{\del x_i}$ are non-decreasing in $x_i$ for all $i$. Hence, $$h_{\rho,\ccal S}(q_1,\dots,p_M) \geq h_{\rho,\ccal S}(p_1,\dots,p_M).$$ If $1 \notin \ccal I_p$, we have $q_1 \geq p_1$, and by the Mean Value Theorem, there exists $c_1 \in [p_1,q_1]$ such that
	$$h_{\rho,\ccal S}(q_1,\dots,p_M)-h_{\rho,\ccal S}(p_1,\dots,p_M) = \frac{\del h_{\rho,\ccal S}}{\del x_1} (c_1,\dots,p_M) \geq \frac{\del h_{\rho,\ccal S}}{\del x_1} (p_1,\dots,p_M) \geq 0.$$
	Thus, the same conclusion holds.
	Proceeding coordinate-wise and using the same argument, we have $$h_{\rho,\ccal S}(p^+) = h_{\rho,\ccal S}(p_1,p_2,\dots,p_M) \leq h_{\rho,\ccal S}(q_1,p_2,\dots,p_M)\leq h_{\rho,\ccal S}(q_1,q_2,\dots,p_M) \leq \dots \leq h_{\rho,\ccal S}(q_1,q_2,\dots,q_M) = h_{\rho,\ccal S}(q^+).$$
\end{proof}
\begin{corollary}\label{cor: convex_risk_functional}
Let $\rho$ be a risk functional such that
$$\sigma_{\rho,\ccal S} = g|_{\Delta^M}$$
for some convex function $g : \RR^{M+1} \to \RR$ with first-order partial derivatives.
Fix $p \in \Delta^M$. Then there exists $\ccal I \subseteq {[M]}_0, \ccal I \neq \emptyset, {[M]}_0$ depending on $p$ such that ${\ccal T}_{\rho,\ccal S, p}^{(1)} \supseteq {\ccal T}_{\ccal S, p}^{(2)}(\ccal I)$.
\end{corollary}
\begin{proof}[Proof of Corollary~\ref{cor: convex_risk_functional}]
	Suppose $\rho$ is a risk functional such that
$$\sigma_{\rho,\ccal S} = g|_{\Delta^M}$$
for some convex function $g : \RR^{M+1} \to \RR$.
Define the function $h_{\rho,\ccal S} : \RR^M \to \RR$ by
$$h_{\rho,\ccal S}(x) = g\paren{1- \sum_{i=1}^M x_i, x}.$$
 Then whenever $\paren{ 1-\sum_{i=1}^M x_i, x} \in \Delta^M$, since $\sigma_{\rho,\ccal S} = g|_{\Delta^M}$,
$$h_{\rho,\ccal S}(x) = g\paren{1- \sum_{i=1}^M x_i, x} = \sigma_{\rho,\ccal S}\paren{1-\sum_{i=1}^M x_i, x}.$$
Furthermore, 
$h_{\rho,\ccal S}$ is convex,
since for any $t \in [0,1], x,y \in \RR^M$, by the convexity of $g$,
\begin{align*}
h_{\rho,\ccal S}(tx +(1-t)y) &= g\paren{1-\sum_{i=1}^M (t x_i+(1-t)y_i),tx+(1-t)y}\\
&= g\paren{t\paren{1-\sum_{i=1}^M x_i} + (1-t)\paren{1-\sum_{i=1}^M y_i},tx+(1-t)y}\\
&\leq tg\paren{1 - \sum_{i=1}^M x_i, x} + (1-t)g\paren{1 - \sum_{i=1}^M y_i, y}\\
&= th_{\rho,\ccal S}(x) + (1-t)h_{\rho,\ccal S}(y).
\end{align*}
Thus, $h_{\rho,\ccal S}$ is convex and has non-decreasing first-order partial derivatives, and the result follows from Proposition~\ref{prop: convex_risk_functional}. 
\end{proof}
\begin{proof}[Proof of Proposition~\ref{prop: table_of_dominant}]
	The result follows from Corollaries~\ref{cor: drf_dominant} and~\ref{cor: convex_risk_functional}.
\end{proof}
In Figure~\ref{fig: dom_rf_venn}, we illustrate how the different classes of risk functionals introduced related to one another. The labels $\mathrm{DRF}$, $\mathrm{Con}$, $\mathrm{Dom}$, and $\mathrm{Low}$ denote the classes of risk functionals that satisfy Proposition~\ref{prop: table_of_dominant}(a), Proposition~\ref{prop: table_of_dominant}(b), Definition~\ref{def: dominant_risk_functional}, and Lemma~\ref{lem: lower_bound} respectively.
\begin{figure}
\begin{center}
	\begin{tikzpicture}
		\draw (2,.75) arc (0:360:1);
		\draw (3,.75) arc (0:360:1);
		\draw (3.75,1) arc (0:360:2.25);
		\draw (4.5,1) arc (0:360:3);
		\draw (0,1.75) node {$\mathrm{DRF}$};
		\draw (3,1.75) node {$\mathrm{Con}$};
		\draw (3.25,3) node {$\mathrm{Dom}$};
		\draw (3.75,3) node[anchor=south west] {$\mathrm{Low}$};
	\end{tikzpicture}
\end{center}
\caption{A classification of the risk functionals that satisfy Lemma~\ref{lem: lower_bound}.}
\label{fig: dom_rf_venn}
\end{figure}
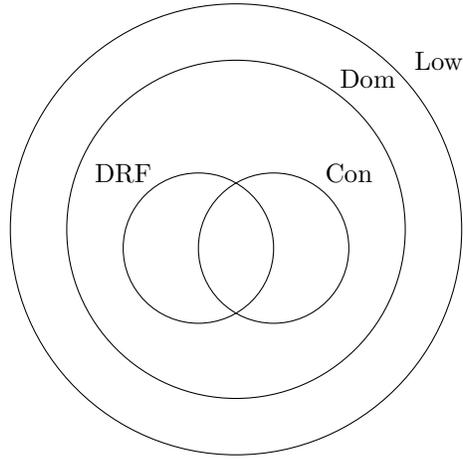
\section{Proof of Lower Bound}
\label{app: lower_bound}
\begin{proof}[Proof of Theorem~\ref{thm: rho_lower_bound}]
	The proof is almost identical to that of \citet[Theorem 3.1]{baudry2020thompson}, by replacing $(\mathrm{CVaR}_\alpha,c^*)$ with $(\rho,r^*)$, $r^* = \max_{k \in [K]}{\rho(\nu_k)}$, so we include it simply for the sake of completeness. Fix $\nu = (\nu_1,\dots,\nu_K) \in {\ccal Q}$, and let $k$ be a sub-optimal arm in $\nu$, that is, $$\rho(\nu_k) < \rho(\nu_{k^*}) =: r^*,$$ where $k^* \in \argmax_{k \in [K]} \rho(\nu_k)$. Suppose there exists $\nu_k' \in {\ccal Q}_k$ such that $$\rho(\nu_k') > r_{k^*}.$$ If this does not hold, ${\ccal K}_{\inf}^\rho (\nu_k, r^*) = +\infty$, and the lower bound holds trivially. Consider the alternative bandit model $\nu'$ in which $\nu_i' = \nu_i$ for all $i \neq k$. By the fundamental inequality (6) of \citet{doi:10.1287/moor.2017.0928}, we obtain that $$\EE_{\pi,\nu}[T_k(n)] \KL(\nu_k,\nu_k') \geq \mathrm{kl}\paren{\EE_{\pi,\nu}\parenb{\frac{T_k(n)}{n}},\EE_{\pi,\nu'}\parenb{\frac{T_k(n)}{n}}},$$ where $\mathrm{kl}(x,y) = x \log(x/y) + (1-x) \log ((1-x)/(1-y))$ denotes the binary relative entropy. By the arguments in \citet{doi:10.1287/moor.2017.0928}, we have $$\liminf_{n \to \infty} \frac{\mathrm{kl}\paren{\EE_{\pi,\nu}\parenb{\frac{T_k(n)}{n}},\EE_{\pi,\nu'}\parenb{\frac{T_k(n)}{n}}}}{\log n } \geq 1$$
	which implies that $$ \liminf_{n \to \infty} \frac{\EE_{\pi,\nu}[T_k(n)]}{\log n} \geq \frac{1}{\KL(\nu_k,\nu_k')}.$$
	Taking the infimum over $\nu_k' \in {\ccal Q}_k$ such that $\rho(\nu_k') \geq r^*$ yields the result, by the definition of ${\ccal K}_{\inf}^\rho$.
\end{proof}

\section{Proof of Upper Bound for $\rho$-MTS}
\label{app: rho_mts}
We begin by listing and recapitulating some notation for the proof of the regret bound. Let $(\ccal C,d)$ be a compact metric space. Let $C \subseteq [0,1]$ denote the common support for all probability measures $\nu_k \in \ccal C, k \in [K]$. Denote $r_k^{\rho} := \rho(\nu_k)$, and $\KL(\mu,\eta)$ the KL divergence between the probability measures $\mu,\eta \in {\ccal C}$. Suppose $(\ccal C,d) = ({\ccal P}_{\ccal S},D_\infty)$ for some finite alphabet $\ccal S=\{s_0,s_1,\dots,s_M\}\subset C$ and $|\ccal S|=M+1$.
\begin{itemize}
	\item Let $T_k^j(t) = \sum_{\ell = 1}^t \II\{A_{\ell} = k, X_k = j\}$ denote the number of times that the arm $k$ is chosen, and gives a reward $s_j$.
	\item Let $\nu_k = {\frak D}_{\ccal S}(p_k)$ denote the multinomial distribution of each arm $k$ characterised by its probability vector $p_k \in \Delta^M$.
	\item Let $\Dir(\alpha_k(t))$ denote a Dirichlet posterior distribution of arm $k$ given the observation after $t$ rounds, where $$ \alpha_k(t) := (1 + T_k^0(t),\dots,1+T_k^{M}(t))$$ characterises its distribution.
	Thus, we can denote the index policy of $\rho$-MTS at time $t$ by $$r_{k}^{\rho}(t) = \rho({\frak D}_{\ccal S}(L_k(t-1))),\quad \text{where}\quad L_k(t-1) \Fol \Dir(\alpha_k(t-1)).$$
	Finally, let $\widehat p_k(t):= \alpha_{k}(t-1)/(T_k(t-1) + M + 1)$ denote the mean of this Dirichlet distribution, and $\widehat \nu_k(t) := {\frak D}_{\ccal S}(\widehat p_k(t))$ the corresponding empirical distribution.
\end{itemize}
\begin{proof}[Proof of Lemma~\ref{lem: upper_bounding_post_cv}]
Following \citet{baudry2020thompson} but replacing $(c_{k,t}^\alpha, c_1^\alpha)$ with $(r_{k}^{\rho}(t), r_1^{\rho})$ therein, for any constant $T_0(n)$,
$$A \leq T_0(n) + A',$$
where
\begin{align*}
A' &\leq \sum_{t = 1}^n \EE\parenb{\II(T_k(t-1) \geq T_0(n), D_\infty({\frak D}_{\ccal S}(\widehat p_k(t)),{\frak D}_{\ccal S}(p_k)) \leq \eeps_2) \cdot \underbrace{\PP(\rho({\frak D}_{\ccal S}(L_k(t-1))) \geq r_1^{\rho} - \eeps_1 \mid {\ccal F}_{t-1})}_{(\dagger)}}.
\end{align*}
Since $\rho$ is continuous, by Lemma~\ref{lem: upper_bound},
$$(\dagger) \leq C_1{(T_k(t-1)+M_k+1)}^{M/2} \exp(-(T_k(t-1)+M+1){\ccal K}_{\inf}^{\rho_{\ccal S}}({\frak D}_{\ccal S}(\widehat{p}_k(t)), r_1^\rho - \eeps_1)).$$
Fix $\eeps_3 \in (0, {\ccal K}_{\inf}^{\rho_{\ccal S}}({\frak D}_{\ccal S}(\widehat{p}_k(t)),r_1^\rho))$.
Following the arguments in \citet{baudry2020thompson}, but using only the lower-semicontinuity of ${\ccal K}_{\inf}^{\rho_{\ccal S}}$ in both arguments, there exists a finite constant $C_1'>0$ such that
$$(\dagger) \leq C_1' \exp(-(T_0(n)+M+1)({\ccal K}_{\inf}^{\rho_{\ccal S}}(\nu_k,r_1^\rho) - \eeps_3/4)),$$
and thus
\begin{align*}
A &\leq T_0(n) + \sum_{t=1}^n C_1' \exp(-(T_0(n)+M+1)({\ccal K}_{\inf}^{\rho_{\ccal S}}(\nu_k,r_1^\rho) - \eeps_3/2))\\
&\leq T_0(n) +C_1' n\exp(-(T_0(n)+M+1)({\ccal K}_{\inf}^{\rho_{\ccal S}}(\nu_k,r_1^\rho) - \eeps_3/2)).
\end{align*}
Choosing
$$T_0(n)=\frac{\log n}{{\ccal K}_{\inf}^{\rho_{\ccal S}}(\nu_k,r_1^\rho) - \eeps_3/2}-(M+1)$$
yields the upper bound
$$A \leq \frac{\log n}{{\ccal K}_{\inf}^{\rho_{\ccal S}}(\nu_k,r_1^\rho) - \eeps_3/2} + \ccal O(1).$$
\end{proof}
\begin{proof}[Proof of Lemma~\ref{lem: upper_bounding_pre_cv}]
To upper bound $B$, we consider the following decomposition:
\begin{align*}
B \leq \underbrace{\EE \parenb{\sum_{t=1}^n \II(A_t=k, r_{k}^{\rho}(t-1) < r_1^{\rho} - \eeps_1)}}_{B_1} + \underbrace{\EE \parenb{\sum_{t=1}^n \II(A_t=k, D_\infty({\frak D}_{\ccal S}(\widehat{p}_k(t)),{\frak D}_{\ccal S}(\nu_k))}}_{B_2}.
\end{align*}
Since $B_2$ does not involve $\rho$, we can upper-bound it directly by the upper bound derived by \citet{pmlr-v117-riou20a} to get, for any $\eeps_2 > 0$, $$B_2 \leq KM\paren{\frac{2M}{\eeps_2} + \frac{2}{\eeps_2^2}}.$$
To upper-bound $B_1$, we follow the arguments in \citet{baudry2020thompson} but replacing $(c_{k,t}^\alpha, c_1^\alpha)$ with $(r_{k}^{\rho}(t), r_1^{\rho})$ to get
\begin{align*}
B_1
&\leq \sum_{\ell=1}^n \EE\parenb{\frac{\PP(\sigma_{\rho,\ccal S}(L_1(\ell-1)) \leq r_1^{\rho} - \eeps_1)}{1- \PP(\sigma_{\rho,\ccal S}(L_1(\ell-1)) \leq r_1^{\rho} - \eeps_1)}}.
\end{align*}
We now partition the right-hand side into the different possible values that $\alpha/(\ell+M+1)$ can take by upper bounding $B_1$ by the sum of three terms $B_{11}$, $B_{12}$, and  $B_{13}$, where
\begin{align*}
B_{11} &= \sum_{\ell=1}^n \EE\parenb{\frac{\PP(\sigma_{\rho,\ccal S}(L_1(\ell-1)) \leq r_1^{\rho} - \eeps_1)}{1- \PP(\sigma_{\rho,\ccal S}(L_1(\ell-1)) \leq r_1^{\rho} - \eeps_1)} \II\{\sigma_{\rho,\ccal S}(\alpha/(\ell+M+1)) \geq r_1^\rho - \eeps_1/2\}}\\
B_{12} &= \sum_{\ell=1}^n \EE\parenb{\frac{\PP(\sigma_{\rho,\ccal S}(L_1(\ell-1)) \leq r_1^{\rho} - \eeps_1)}{1- \PP(\sigma_{\rho,\ccal S}(L_1(\ell-1)) \leq r_1^{\rho} - \eeps_1)} \II\{r_1^\rho - \eeps_1 \leq \sigma_{\rho,\ccal S}(\alpha/(\ell+M+1)) \leq r_1^\rho - \eeps_1/2\}}\\
B_{13} &= \sum_{\ell=1}^n \EE\parenb{\frac{\PP(\sigma_{\rho,\ccal S}(L_1(\ell-1)) \leq r_1^{\rho} - \eeps_1)}{1- \PP(\sigma_{\rho,\ccal S}(L_1(\ell-1)) \leq r_1^{\rho} - \eeps_1)} \II\{\sigma_{\rho,\ccal S}(\alpha/(\ell+M+1)) \leq r_1^\rho - \eeps_1\}}.	
\end{align*}
We now upper bound each of these terms by  constants.
\subsubsection{\underline{Case 1: $\sigma_{\rho,\ccal S}(\alpha/(\ell+M+1)) \geq r_1^\rho - \eeps_1/2$}}
\phantom{-}\\
Since $\sigma_{\rho,\ccal S}$ is continuous, there exists $p_* \in \Delta^M$ such that
$$\KL({\frak D}_{\ccal S}(\alpha/(\ell+M+1)),{\frak D}_{\ccal S}(p_*))={\ccal K}_{\inf}^{\rho_{\ccal S}}({\frak D}_{\ccal S}(\alpha/(\ell+M+1)),r_1^\rho-\eeps_1).$$
By Lemma~\ref{lem: upper_bound}, under the condition $\sigma_{\rho,\ccal S}(\alpha/(\ell+M+1)) \geq r_1^\rho - \eeps_1/2$,
\begin{align*}
\PP(\sigma_{\rho,\ccal S}(L_1(\ell-1))
\leq r_1^{\rho} - \eeps_1)
&\leq C_1{(\ell+M+1)}^{M/2} \exp(-(\ell+M+1)\KL({\frak D}_{\ccal S}(\alpha/(\ell+M+1)),{\frak D}_{\ccal S}(p_*))).
\end{align*}
Define $$\delta = \inf \{\KL(\eta,\mu) : \mu \in \inv{\rho}((-\infty,r_1^{\rho}-\eeps_1]), \eta \in \inv{\rho}([r_1^{\rho}-\eeps/2,\infty))\}.$$
Following the computation in \citet{baudry2020thompson}, $\delta > 0$ and
there exists $\ell_1,\ell_\gamma \in \NN$ such that for $\ell > \max\{\ell_1,\ell_\gamma\}$,
\begin{align*}
B_{11} 
&\leq \max\{\ell_1,\ell_\gamma\} + \sum_{\ell=\max\{\ell_1,\ell_\gamma\}}^n \gamma C_1{(\ell+M+1)}^{M/2} \exp(-(\ell+M+1)\delta) \leq {\ccal O}(1).
\end{align*}
\subsubsection{\underline{Case 2: $r_1^\rho - \eeps_1 \leq \sigma_{\rho,\ccal S}(\alpha/(\ell+M+1)) \leq r_1^\rho - \eeps_1/2$}}
\phantom{-}\\
Under the condition $\sigma_{\rho,\ccal S}(\alpha/(\ell+M+1)) \geq r_1^\rho - \eeps_1$, $${\ccal K}_{\inf}^{{\rho}_{\ccal S}}({\frak D}_{\ccal S}(\alpha/(\ell+M+1)),r_1^\rho-\eeps_1)=0.$$
Since $\rho$ is assumed to be dominant, we can apply Lemma~\ref{lem: lower_bound} to get
\begin{align*}
\PP(\sigma_{\rho,\ccal S}(L_1(\ell-1))
\geq r_1^{\rho} - \eeps_1)
&\geq C_2 \cdot \frac{\exp(-(\ell+M+1) {\ccal K}_{\inf}^{{\rho}_{\ccal S}}({\frak D}_{\ccal S}(\alpha/(\ell+M+1)),r_1^\rho-\eeps_1))}{{(\ell+M+1)}^{\frac{3M}{2}+1}}\\
&\geq \frac{C_2}{{(\ell+M+1)}^{\frac{3M}{2}+1}}.
\end{align*}
Hence,
\begin{align*}
B_{12} 
&\leq \sum_{\ell=1}^n \EE\parenb{\frac{\PP(\sigma_{\rho,\ccal S}(L_1(\ell-1)) \leq r_1^{\rho} - \eeps_1)}{1- \PP(\sigma_{\rho,\ccal S}(L_1(\ell-1)) \leq r_1^{\rho} - \eeps_1)} \II\{r_1^\rho - \eeps_1 \leq \sigma_{\rho,\ccal S}(\alpha/(\ell+M+1)) \leq r_1^\rho - \eeps_1/2\}}\\
&\leq \sum_{\ell=1}^n \EE\parenb{1 \cdot C_2^{-1} {(\ell+M+1)}^{\frac{3M}{2}+1} \II\{r_1^\rho - \eeps_1 \leq \sigma_{\rho,\ccal S}(\alpha/(\ell+M+1)) \leq r_1^\rho - \eeps_1/2\}}\\
&\leq \sum_{\ell=1}^n C_2^{-1} {(n+M+1)}^{\frac{3M}{2}+1} \PP(\sigma_{\rho,\ccal S}(\alpha/(\ell+M+1)) \leq r_1^\rho - \eeps_1/2).
\end{align*}
Since $\rho$ is continuous on ${\ccal P}_{\ccal S}$, by the DKW inequality in Lemma~\ref{lem: dkw_inequality}, there exists $\delta_{\eeps_1/3} > 0$ such that for any $m \in \NN$,
$$\PP\paren{\parenl{\sigma_{\rho,\ccal S}\paren{\frac{\alpha-1}{m}}-r_1^\rho} > \eeps_1/3} \leq 2e^{-m\delta_{\eeps_1/3}^2/2}.$$
By Proposition~\ref{prop: continuity_rho}, there exists $\ell_2 \in \NN$ such that for $m \geq \ell_2$,
$$\parenl{\sigma_{\rho,\ccal S}\paren{\frac{\alpha}{m+M+1}}-\sigma_{\rho,\ccal S}\paren{\frac{\alpha-1}{m}}} < \eeps_1/6.$$
Hence, for $\ell \geq \ell_2$,
$$\PP(\sigma_{\rho,\ccal S}(\alpha/(\ell+M+1)) \leq r_1^\rho - \eeps_1/2) \leq \PP(\sigma_{\rho,\ccal S}((\alpha-1)/\ell) \leq r_1^\rho - \eeps_1/3) \leq 2e^{-\ell \delta_{\eeps_1/3}^2/2}.$$
Consequently,
$$B_{12} \leq 2C_2^{-1} \sum_{\ell=1}^n {(\ell+M+1)}^{\frac{3M}{2}+1} e^{-\ell \delta_{\eeps_1/3}^2/2} \leq {\ccal O}(1).$$

\subsubsection{\underline{Case 3: $\sigma_{\rho,\ccal S}(\alpha/(\ell+M+1)) \leq r_1^\rho - \eeps_1$}}
\phantom{-}\\
Since $\rho$ is dominant, we can apply Lemma~\ref{lem: lower_bound} to get
\begin{align*}
\PP(\sigma_{\rho,\ccal S}(L_1(\ell-1))
\geq r_1^{\rho} - \eeps_1)
&\geq C_2 \cdot \frac{\exp(-(\ell+M+1) {\ccal K}_{\inf}^{{\rho}_{\ccal S}}({\frak D}_{\ccal S}(\alpha/(\ell+M+1)),r_1^\rho-\eeps_1))}{{(\ell+M+1)}^{\frac{3M}{2}+1}}\\
&=: C_2 \frac{1}{{(\ell+M+1)}^{\frac{3M}{2}+1} f(\alpha,\ell,r_1^\rho - \eeps_1)}.
\end{align*}
Hence,
\begin{align*}
B_{13} 
&\leq \sum_{\ell=1}^n C_2^{-1} {(\ell+M+1)}^{\frac{3M}{2}+1} \EE[f(\alpha,\ell,r_1^\rho -\ \eeps_1) \cdot \II\{\sigma_{\rho,\ccal S}(\alpha/(\ell+M+1)) \leq r_1^\rho - \eeps_1\}]
\end{align*}
Using Lemma~\ref{prop: continuity_rho} and the continuity of the KL divergence, for any $\eeps'>0$, there exists $\ell_{\eeps'} \in \NN$ such that for $\ell \geq \ell_{\eeps'}$,
$$\ell \KL((\alpha-1)/\ell,p) \geq (\ell+M+1)(\KL(\alpha/(\ell+M+1))-\eeps').$$
Define
$$\ccal A := \{\alpha \in {\{1,\dots,\ell+1\}}^M : \sigma_{\rho, \ccal S}(\alpha/(\ell+M+1)) < r_1^\rho - \eeps_1\}.$$
 By the computations in \citet{baudry2020thompson},
\begin{align*}
&\EE[f(\alpha,\ell,r_1^\rho -\ \eeps_1) \cdot \II\{\sigma_{\rho,\ccal S}(\alpha/(\ell+M+1)) \leq r_1^\rho - \eeps_1\}]\\
&\leq \sum_{\alpha \in \ccal A} \exp\paren{-(\ell+M+1)({\ccal K}_{\inf}^{\rho_{\ccal S}}({\frak D}_{\ccal S}(\alpha/(\ell+M+1)),r_1^\rho)-{\ccal K}_{\inf}^{\rho_{\ccal S}}({\frak D}_{\ccal S}(\alpha/(\ell+M+1)),r_1^\rho-\eeps_1) - \eeps')}\\
&\leq \sum_{\alpha \in \ccal A} \exp\paren{-(\ell+M+1)(\delta' - \eeps')}.
\end{align*}
where by Proposition~\ref{prop: continuity_props} applied to the arguments in \citet{baudry2020thompson},
$$\delta' := \inf_{\mu \in \inv{\rho_{\ccal S}}((-\infty,r_1^\rho-\eeps_1])} \{{\ccal K}_{\inf}^{\rho_{\ccal S}}(\mu,r_1^\rho)-{\ccal K}_{\inf}^{\rho_{\ccal S}}(\mu,r_1^\rho-\eeps_1)\} > 0.$$
Finally, choosing $\eeps' = \delta'/2$ and using the fact that $|\ccal A| \leq {(\ell+1)}^{M+1}$ yields
$$B_{13} \leq \ell_{\delta'/2} + \sum_{\ell=\ell_{\delta'/2}+1}^n C_2^{-1} {(\ell+M+1)}^{\frac{3M}{2}+1} {(\ell+1)}^{M+1} \exp(-(\ell+M+1)\delta'/2) \leq \ccal O(1).$$
Consolidating,
$B_1 \leq B_{11} + B_{12} + B_{13} \leq \ccal O(1)$.
\end{proof}
\end{document}